\documentclass{article}

\usepackage{arxiv}

\usepackage[utf8]{inputenc} 
\usepackage[T1]{fontenc}    
\usepackage{hyperref}       
\usepackage{url}            
\usepackage{booktabs}       
\usepackage{amsfonts}       
\usepackage{nicefrac}       
\usepackage{microtype}      
\usepackage{xcolor}         

\usepackage{amsmath}
\usepackage{amssymb}
\usepackage{mathtools}
\usepackage{amsthm}
\usepackage[export]{adjustbox}
\usepackage{tikz-cd}
\usepackage{mathabx,epsfig}
\usepackage{quiver}
\usepackage{amsmath}
\usepackage{mathrsfs}
\usepackage{bbm}

\usepackage[capitalize,noabbrev]{cleveref}

\theoremstyle{plain}
\newtheorem{theorem}{Theorem}[section]
\newtheorem{proposition}{Proposition}[section]

\newtheorem{corollary}{Corollary}[section]
\newtheorem{example}{Example}[section]
\theoremstyle{definition}
\newtheorem{definition}{Definition}[section]

\theoremstyle{remark}
\newtheorem{remark}{Remark}[section]

\usepackage{amsmath,amsfonts,bm}

\def\eqref#1{equation~\ref{#1}}

\def\1{\bm{1}}

\DeclareMathAlphabet{\mathsfit}{\encodingdefault}{\sfdefault}{m}{sl}
\SetMathAlphabet{\mathsfit}{bold}{\encodingdefault}{\sfdefault}{bx}{n}

\def\gC{{\mathcal{C}}}
\def\gD{{\mathcal{D}}}

\def\gF{{\mathcal{F}}}
\def\gG{{\mathcal{G}}}

\def\gM{{\mathcal{M}}}
\def\gN{{\mathcal{N}}}

\def\sN{{\mathbb{N}}}

\def\sR{{\mathbb{R}}}

\def\sZ{{\mathbb{Z}}}

\usepackage{bbm}
\usepackage{wrapfig}
\usepackage{mathabx,epsfig}

\makeatletter
\newcommand*{\transpose}{%
  {\mathpalette\@transpose{}}%
}
\newcommand*{\@transpose}[2]{%
  \raisebox{\depth}{$\m@th#1\intercal$}%
}
\usepackage{comment} 
\usepackage[dvipsnames]{xcolor}

\usepackage{natbib} 
\PassOptionsToPackage{sort}{natbib}
\definecolor{darkgreen}{rgb}{0.0, 0.5, 0.13}

\definecolor{first}{RGB}{68, 1, 64} 
\definecolor{second}{RGB}{49, 104, 142} 
\definecolor{third}{RGB}{53, 183, 121} 
\definecolor{fourth}{RGB}{253, 231, 37}

\newcommand{\rmapf}[2]{\mathcal{F}_{#1 \unlhd #2}}
\newcommand{\rmapc}[2]{\mathcal{C}_{#1 \unlhd #2}}
\newcommand{\rmapm}[2]{\mathcal{M}_{#1 \unlhd #2}}
\newcommand{\psd}{\mathcal{S}_+}
\newcommand{\concat}[2]{#1\mathbin\Vert#2}
\newcommand{\gauss}[1]{\mathcal{G}(\mathbb{R}^{#1})}
\def\acts{\mathrel{\reflectbox{$\righttoleftarrow$}}}
\newcommand{\values}[2]{#1_{\pm #2}}
\newcommand{\valuesf}[2]{{\color{ForestGreen}\mathbf{#1}}_{\pm #2}}
\newcommand{\valuess}[2]{{\color{DarkOrchid}\mathbf{#1}}_{\pm #2}}
\newcommand{\Conv}{%
  \mathop{\scalebox{1.5}{\raisebox{-0.2ex}{$\circledast$}}
  }
}

\usepackage{enumitem}
\definecolor{lb}{RGB}{31,119,180}
\usepackage{tcolorbox}
\newtcolorbox{mybox}[1]{colback=lb!1!white,colframe=lb!70!black,fonttitle=\bfseries,title=#1,
left=3pt,    
  right=3pt,   
}

\newtcolorbox{myboxgray}[1]{colback=gray!15,colframe=black!70,fonttitle=\bfseries,title=#1, %
left=3pt,    
  right=3pt,   
}

\usepackage{amsmath}
\usepackage{multirow}
\usepackage{tabularx}
\usepackage[table]{xcolor}

\title{Gaussian Sheaf Neural Networks}

\author{%
  André Ribeiro \\
  Getulio Vargas Foundation\\
  \texttt{andre.guimaraes@fgv.br} \\
   \And
   Ana Luiza Tenório \\
   Getulio Vargas Foundation \\
   \texttt{ana.tenorio@fgv.br} \\   
    \And
   Tiago da Silva \\
   MBZUAI \\
   \texttt{tiago.dasilva@mbzuai.ac.ae} \\   
      \And
   Diego Mesquita \\
   Getulio Vargas Foundation \\
   \texttt{diego.mesquita@fgv.br} 
}

\begin{document}

\vspace{-9pt}
\maketitle
\vspace{-9pt}

\begin{abstract}
Graph Neural Networks (GNNs) have become the de facto standard for learning on relational data. While traditional GNNs' message passing is well suited for vector-valued node features, there are cases in which node features are better represented by probability distributions than real vectors. Concretely, when node features are Gaussians, characterized by a mean and a covariance matrix, naively concatenating their parameters into a single vector and applying standard message passing discards the geometric and algebraic structure that governs means and covariances. We propose Gaussian Sheaf Neural Networks (GSNNs), a principled framework that incorporates these inductive biases into graph-based learning. Building on the theory of cellular sheaves, we derive a new Laplacian operator that generalizes the sheaf Laplacian to this setting and preserves its key properties. We complement our theoretical contributions with experiments on synthetic and real-world data that illustrate the practical relevance of GSNNs.
\end{abstract}

\section{Introduction}

Graph neural networks (GNNs) are the \emph{de facto} standard for machine learning over networked data, with applications in molecular modeling~\citep{duvenaud2015convolutional,gilmer2017neural}, recommender systems~\citep{ying2018graph}, and physics simulation~\citep{sanchez2020learning}. While traditional message passing is well suited for vector-valued node features, many problems are more naturally posed in terms of \emph{distributions} 
attached to nodes: uncertainty-aware predictions on traffic or sensor networks, knowledge-graph entities equipped with confidence, nodes representing populations or cohorts, and Gaussian-process-style 
features over relational domains, among others. A canonical and particularly tractable instance is the case in which each node is associated with a Gaussian distribution, characterized by a mean and  a covariance matrix. Although their parameters can be flattened into 
a vector, doing so leaves on the table the geometric and algebraic structure that governs means and covariances --- structure that is highly informative as an inductive bias for learning. In this work, we observe that the recently introduced cellular sheaves provide a natural way to encode this structure in graph-based learning.

Cellular sheaves are simplified versions of sheaves defined over a graph, associating (i) a vector space, the \emph{stalk}, to each vertex and each edge, and (ii) a linear map, the \emph{restriction map}, to each incident vertex--edge pair. These constructs induce a 
continuous-time diffusion process over the graph that can be discretized into neural architectures for relational data. Early sheaf-based models~\citep{hansen2020sheaf} relied on hand-crafted 
restriction maps, while \citet{bodnar2022neural} proposed learning them end-to-end -- yielding models that avoid oversmoothing and adaptively control the Dirichlet energy, even when adequate restriction maps are not known beforehand.

Crucially, the case for sheaf-based architectures fits within a broader program in geometric deep learning that argues for respecting the intrinsic 
geometry of the data rather than forcing it into a flat Euclidean container~\citep{bronstein2021geometric,bronstein2017geometric}. A recurring lesson from this line of work is that disregarding the  structure of the underlying space --- 
whether it be a manifold of 
hierarchies~\citep{nickel2017poincare,chami2019hyperbolic} or a non-Euclidean space of mixed 
curvature~\citep{gu2018learning,xiong2022pseudo} --- 
incurs distortion and hampers expressivity. Gaussian-valued node features are a clear instance of this phenomenon: their parameter space is not a flat vector space but the statistical 
manifold $\mathcal{N}(\mathbb{R}^d)\cong\mathbb{R}^d\times\psd^d$, admitting only a restricted class of structure-preserving maps. Cellular sheaves, by attaching a stalk and a restriction map to each incidence in the graph, give us precisely the local-to-global language needed to encode such structure as an inductive bias for 
message passing --- and, as we show, this language extends gracefully beyond the vector-space setting in which it was originally formulated.

In this work, we define a Gaussian sheaf in analogy to a cellular sheaf, replacing each vector-space stalk with a space of Gaussian distributions, parameterized by a mean and a covariance. This seemingly small 
modification, however, breaks several of the categorical properties that make the standard sheaf Laplacian well-defined, and recovering 
them requires non-trivial work. Cellular sheaves over $\textbf{Vect}$ form a category equivalent to that of actual sheaves, an equivalence implicitly used 
by~\citet{hansen2020sheaf,bodnar2022neural,barbero2022sheaf} to define the sheaf Laplacian via the (coboundary) operators of sheaf cohomology. Replacing $\textbf{Vect}$ with a category $\textbf{Gauss}$ 
of Gaussian distributions raises the question of whether $\textbf{Gauss}$ is well-behaved enough for the same construction to go through. We show that it is, despite the additional difficulty that $\textbf{Gauss}$ is not \emph{abelian} --- so cohomological constructions like the coboundary map (a generalization of the incidence matrix) must draw on tools from non-abelian settings~\citep{patchkoria2006exactness,JUN2017306}. We elaborate on this in Section~\ref{sec:gcs} and provide further details in Appendix~\ref{sec:app:cohomology}. We summarize our main contributions in \Cref{tab:summaries}.

\begin{table}[!t]
\centering
\small
\caption{Summary of our main contributions.}
\label{tab:summaries}
\vspace{4pt}
\renewcommand{\arraystretch}{1.35}
\begin{tabular}{p{.68\textwidth} l}
\toprule
\textbf{Contribution} & \textbf{Reference} \\
\midrule
    Construction of the sheaf of means ($\mathbf{Vect}$) and covariances ($\mathbf{SDef}$) & \cref{def:lap_sheaf_covariances} \\ 
    Characterization of $\mathbf{SDef}$'s Laplacian ($L_{\mathcal{C}}$)  & \cref{thm:sdefLap} \\ 
    Design of the Gaussian sheaf ($\equiv \!\! \mathbf{Vect} \! \times\!  \mathbf{SDef}$) \& its Laplacian ($L_{\mathcal{F}}$) &  \cref{def:gauss_sheaf,def:lap_sheaf_gaussian} \\
    Expressivity of the Gaussian sheaf for diverse restriction map classes & Propositions \ref{prop:orbitorth}--\ref{prop:orbitgen} \\ 
    Introduction of GSNN for DoD regression on networked data & \cref{sec:GSNN} \\ 
    GSNN often outperforms baselines in both synthetic and real-world datasets & \cref{tab:results,tab:adapted_graphLap} \\ 
    GSNN is robust against oversmoothing as model depth increases & \cref{prop:lyapenergy,tab:A} \\
\bottomrule
\end{tabular}
\end{table}


\section{Notation and Background}\label{sec:prelim} 
A graph is a mathematical structure represented as $G = (V, E)$, where $V$ is the set of nodes and $E$ is the set of edges. In this paper, we consider undirected, connected, and unweighted graphs. We use $|S|$ to denote the cardinality of any finite set $S$. Our theoretical developments rely on basic notions of category theory, for which we provide a brief self-contained introduction in \cref{sec:app:catshv}.

\subsection{Cellular sheaves on a graph}
We start recalling the main structures in the Neural Sheaf Diffusion (NSD) model.

\begin{definition}\label{def:cellsheaf}
    A \textbf{cellular sheaf} $(G, \mathcal{F})$ over a (undirected) graph $G = (V, E)$ associates:\looseness=-1 
\begin{enumerate}[leftmargin=16pt, itemsep=1pt, nosep]
        \item Vector spaces $\mathcal{F}(v)$ and $\mathcal{F}(e)$ to each vertex $v \in V$ and each edge $e \in E$, called \textbf{stalks}.
        \item Linear maps $\rmapf{v}{e} : \gF(v) \to \gF(e)$ to incident vertex-edge pair $v \unlhd e$, called \textbf{restriction maps}.
    \end{enumerate}
\looseness=-1
\end{definition}
In this work, the above definition is understood as a \textbf{Vect}-valued sheaf. Formally, \textbf{Vect} is a \textit{category} whose objects are finite-dimensional vectors spaces and whose morphism are linear maps. More generally, a \textbf{D}-valued sheaf is a sheaf whose stalks and restrictions maps are in \textbf{D}.

Cellular sheaves allow us to introduce a generalized Laplacian that enhances the performance of GNNs, but first we describe the domain and codomain of a coboundary operator, $\delta$, which plays a central role in defining the sheaf Laplacian.

\begin{definition}\label{def:0cochains}
    The \textbf{space of 0-cochains} and the \textbf{space of 1-cochains} of a cellular sheaf $(G, \gF)$ are  
    \begin{align*}
        C^0(G, \gF) = \bigoplus_{v \in V} \gF(v) \mbox{ and }
        C^1(G, \gF) = \bigoplus_{e \in E} \gF(e).
    \end{align*} 
    where $\oplus$ denotes the (external) direct sum.
\end{definition}
To define the coboundary operator $\delta \!:\! C^0(G,\gF) \to C^1(G,\gF)$, we choose an orientation $e \!=\! u \!\to\! v$  for every $e \in E$ and set $(\delta \mathbf{X})_e = \rmapf{v}{e} \mathbf{x}_v - \rmapf{u}{e} \mathbf{x}_u$. The sheaf Laplacian is then defined by $L_{\gF} = \delta^{\top}\delta$, or equivalently by
\begin{equation}\label{eq:lap}
    L_{\mathcal{F}}(\mathbf{X})_v\!:=\!\!\sum_{v, u \unlhd e} \rmapf{v}{e}^{\top}\left(\!\rmapf{v}{e} \mathbf{x}_v\!-\!\rmapf{u}{e} \mathbf{x}_u\right), \forall v \in V.
\end{equation}
This generalizes the usual graph Laplacian: when stalks are all $\sR$ and restriction maps are identities, $\delta^\top$ is the incidence matrix and $L_\gF$ recovers the $n \times n$ graph Laplacian. The diagonal and non-diagonal blocks $L_{vv} = \sum_{v \unlhd e} \rmapf{v}{e}^{\top} \rmapf{v}{e}$ and $L_{uv} = L_{vu}^\top = -\rmapf{v}{e}^{\top} \rmapf{u}{e}$ define a sheaf degree matrix $D$ and a sheaf adjacency $A$, respectively.
We fix a common dimension $d$ for node and edge stalks, so $L_\gF \in \sR^{nd \times nd}$. Sheaf-based networks typically use the normalized form $\Delta_{\gF} = D^{-\frac{1}{2}}L_{\gF}D^{-\frac{1}{2}}$.



\noindent\textbf{Neural Sheaf Diffusion.} Inspired by the Sheaf Neural Network given by the equation $\dot{\mathbf{X}}(t) = -\sigma(\Delta_{\gF(t)}(\mathbf{I} \otimes \mathbf{W}_1)\mathbf{X}_t\mathbf{W}_2)$ \cite{hansen2019toward}, \citet{bodnar2022neural} proposed the NSD model: 
\begin{equation}
    \mathbf{X}_{t+1} = (1+\varepsilon)\mathbf{X}_t - \sigma(\Delta_{\gF(t)}(\mathbf{I} \otimes \mathbf{W}_1^t)\mathbf{X}_t\mathbf{W}_2^t),
\end{equation}
where $W_1$ and $W_2$ are weight matrices, $\otimes$ the Kronecker product, $\sigma$ an activation and $\varepsilon \in [-1,1]^{nd}$. Furthermore, the Laplacian is learned locally through the graph features: each restriction map $F_{v \unlhd e}$ is learned using a function $\Phi: \mathbb{R}^{2d} \to \mathbb{R}^{d \times d}$, with $\gF_{v \unlhd e=(v,u)} = \Phi(\mathbf{x}_v, \mathbf{x}_u)$. 

Finally, we recall that a global section of a cellular sheaf over $G$ is a certain choice of element for each node stalk. Formally, the space of global sections for $(G, \gF)$ is:
\begin{equation*}
        \Gamma(G, \gF) = \{\mathbf{X} \in C^0(G, \gF) \:|\: \rmapf{v}{e}\mathbf{x}_v = \rmapf{u}{e}\mathbf{x}_u \}    
\end{equation*}
It is straightforward to check that  $\Gamma(G, \gF) = \ker L_{\gF} = \ker \delta$. Observing that $\delta$ is a coboundary operator built upon a cohomology theory for cellular sheaves \cite{hansen2019toward,hansen2020laplacians, hansen2021opinion}, we also have $H^0(G, \gF) := \ker \delta$.  \looseness=-1

 \textbf{Problem statement.} We focus on distribution-on-distribution (DoD) regression over networked (i.e., non-iid) data. 
Equivalently, this task can be seen as a special case of node-level prediction in which both the node features and their respective outputs are probability distributions. For mathematical convenience, we assume that input features can be well-approximated by multivariate Gaussian measures. 
\looseness=-1

\section{Gaussian Cellular Sheaves}\label{sec:gcs} 

\begin{wrapfigure}
    [12]{r}{0.5\textwidth}
\vspace{-12pt}
\includegraphics[width=.98\linewidth]{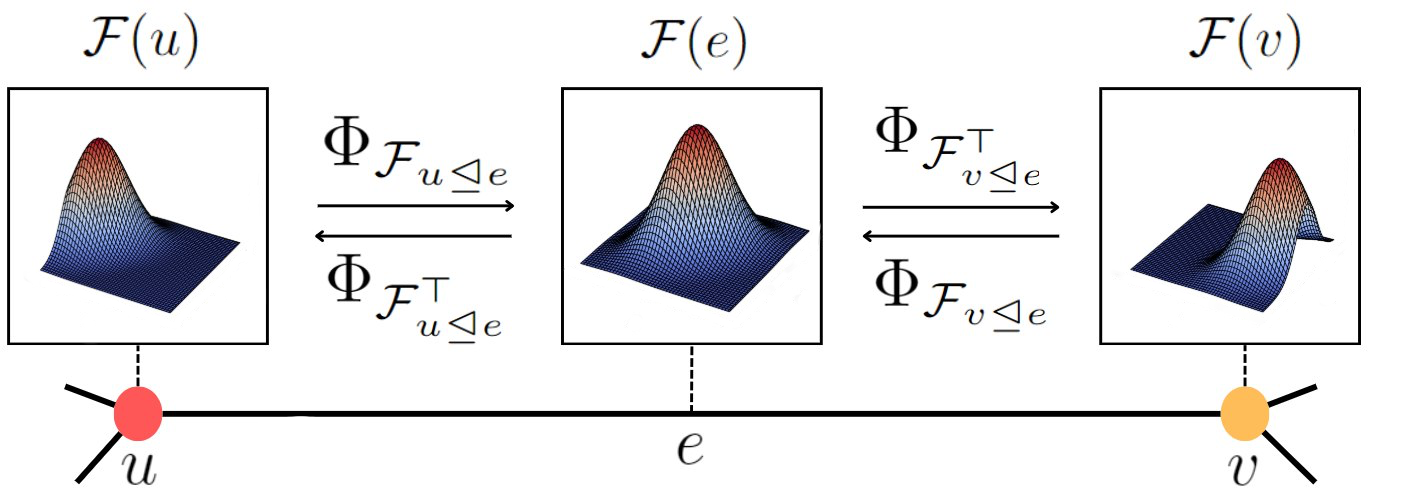}
\vspace*{-2.8mm}
    \caption{A Gaussian sheaf $(G,\gF)$ shown for a single edge $e = (u,v)$ of $G$. The stalks are $\gauss{2} \simeq \sR^2 \!\times\! \psd^2$, the space of Gaussian distributions on $\sR^2$. The distribution features move between the spaces through restriction maps.}
    \label{fig:sheaf}
    \vspace{-25pt}
\end{wrapfigure}
We propose a cellular sheaf suited for graph data whose node features are multivariate normal distributions. Briefly, each stalk of a  Gaussian sheaf is a space of Gaussian distributions, and restriction maps are structure-preserving maps between such spaces. Since  Gaussians are uniquely determined by their parameters (mean and covariance), there is an isomorphism between the space $\gauss{d}$ of Gaussian distributions on $\sR^d$ and $\sR^d \times \psd^{d}$, in which $\psd^{d}$ is the cone of $n$-dimensional positive semidefinite matrices. Figure \ref{fig:sheaf} illustrates the Gaussian sheaf.

In this section, we introduce the sheaf of means whose stalks are $\sR^d$, and the sheaf of covariances whose stalks are $\psd^{d}$, which will be melded together to define Gaussian sheaves. Readers primarily interested in the GSNN architecture can take \cref{def:gauss_sheaf} and the normalized Laplacian (\cref{def:lap_sheaf_covariances}) as the key objects: the categorical machinery developed here ensures these are well-defined.\looseness=-1

\begin{remark}\label{rmk:pushforward}
    For a Gaussian distribution $\nu = \gN(\mu, \Sigma)$ with $\mu \in \sR^d$ and $\Sigma \in \psd^d$, the pushforward of $\nu$ by a matrix $A \in \sR^{d \times d}$ is the Gaussian distribution $A\#\nu = \gN(A\mu, A\Sigma A^{\top})$. This closure under pushforward by linear maps is what inspired us to develop a Gaussian sheaf connecting the mean and covariance sheaves through their restriction maps.
\end{remark}

Observe that the sheaf of means is a standard cellular sheaf, as defined in \cref{def:cellsheaf}. The nomenclature is to reinforce that the vectors in it stalks should be interpreted as a mean in a Gaussian distribution. Thus, we dedicate our attention to introduce the sheaf of covariances, which is  the technically challenging part.

\subsection{Sheaves for covariance matrices}
\label{subsec:gauss_construction}

As an intermediate step towards Gaussian sheaves, we propose the category \textbf{SDef} whose objects are the (convex) cones $\psd^n = \{\Sigma \in \sR^{n\times n} \mid \Sigma = \Sigma^\top, \Sigma \succeq 0\}$ of positive semidefinite matrices, and whose morphisms $\varphi_A : \psd^n \to \psd^m$ act by $\varphi_A(\Sigma) = A\Sigma A^\top$ for $A \in \operatorname{Hom}(\sR^n, \sR^m)$. The family of such morphisms is closed under addition and satisfies $\varphi_B \circ \varphi_A = \varphi_{BA}$. We use the same symbol for a linear map and its matrix in a chosen basis.\looseness=-1
\vspace{-5pt}
\begin{mybox}{}
\vspace{-5pt}
\textbf{The sheaf of covariances.} Let $G = (V,E)$ be an undirected graph and $(G, \gM)$ a sheaf of means with stalk dimension $d$, and assume $\rmapm{v}{e} \in \sR^{d\times d}$ is invertible for every $v \unlhd e$. The sheaf of covariances $(G, \gC)$ has $\psd^{d}$ as stalks (for both nodes and edges), with restriction maps $\rmapc{v}{e} = \varphi_{\rmapm{v}{e}}$ induced by those of $(G, \gM)$.
\vspace{-5pt}
\end{mybox}
To define the cochain spaces of $(G, \gC)$, observe that $\gC(v)$ and $\gC(e)$ are not vector spaces, so direct sums no longer apply. Category theory provides the natural replacement: finite direct sums in $\textbf{Vect}$ correspond to categorical products, and the analogous product in $\textbf{SDef}$ for $\psd^n, \psd^m$ is the cone of block-diagonal PSD matrices $\psd^n \oplus \psd^m = \{\Sigma_1 \oplus \Sigma_2 \mid \Sigma_1 \in \psd^n, \Sigma_2 \in \psd^m\}$, where $\Sigma_1 \oplus \Sigma_2$ denotes the block-diagonal matrix with $\Sigma_1$ and $\Sigma_2$ on the diagonal. We thus define $C^0(G, \gC) = \bigoplus_{v \in V} \gC(v)$ and $C^1(G, \gC) = \bigoplus_{e \in E} \gC(e)$, where $\bigoplus$ is the product in \textbf{SDef}.

We would like to define a coboundary map $\delta_{\gC}$ and recover $H^0(G,\gC) = \ker \delta_{\gC}$, but for $A \neq 0$ the equation $A\Sigma A^\top = 0$ forces $\Sigma = 0$, so kernels in \textbf{SDef} are degenerate. We instead use \emph{equalizers}, which generalize kernels by capturing the set on which two maps agree: for $\varphi_A, \varphi_B : \psd^n \to \psd^m$, $\operatorname{Eq}(\varphi_A, \varphi_B) = \{\Sigma \in \psd^n \mid A\Sigma A^\top = B\Sigma B^\top\}$.  For comparison, in \textbf{Vect}, given a linear map $\varphi$ and the null map $0$, we have $\operatorname{Eq}(\varphi,0) = \ker(\varphi)$. Importantly, the constructions of products and equalizers extend to sub-cones of $\psd^n$; we provide a more thorough treatment in Appendix \ref{sec:app:catshv}.

Now, similarly to works on the (co)homology of semi-modules~\citep{patchkoria2018cohomology} (which, like PSD cones, lack additive inverses and thus admit only trivial kernels), we define $\delta_{\gC}$ in terms of a pair of positive and negative operators $(\delta_{\gC}^+, \delta_{\gC}^-)$ and characterize $H^{0}(G, \gC)$ as the equalizer of both. 
We also define $\delta_{\gC}^+$ and $\delta_{\gC}^-$ as a function of their counterparts $\delta_{\gM}^+$ and $\delta_{\gM}^-$ in $\gM$, where $\delta_{\gM}^+ - \delta_{\gM}^- = \delta_{\gM}$. 
This dependence captures the idea that taking linear transformations of independent random variables has a coordinated impact in both the resulting mean-vector and the covariance matrix (see \Cref{rmk:pushforward}).\looseness=-1
\begin{definition}
    Let $\{\mathbf{e}_i\}_{i \in [|E|]}$ be the standard basis of $\sR^{|E|}$, and consider a family $\{\mathbf{A}_i\}_{i \in [|E|]}$ with $\mathbf{A}_i \!=\! (\mathbf{e}_i^{\phantom{\top}}\!\mathbf{e}_i^{\top})\otimes I_d$.
    We define the operator $\delta_{\gC}\! =\!  \delta_{\gC}^+\! + \delta_{\gC}^-$ s.t. for all $\Sigma \in C^0(G,\gC)$ \looseness=-1
\begin{align*}
    &\delta_{\gC}^+(\Sigma) = \sum_i \varphi_{\mathbf{A}_i^{\phantom{+}}}\!\!\!\left(\varphi_{\delta_{\gM}^+}(\Sigma)\right) \quad \text{ and } \quad
    \delta_{\gC}^-(\Sigma) = \sum_i \varphi_{\mathbf{A}_i^{\phantom{-}}}\!\!\!\left(\varphi_{\delta_{\gM}^-}(\Sigma)\right).
\end{align*}    
\looseness=-1
\end{definition}
\vspace{-8pt}
The family of $\mathbf{A}_i$ guarantees that we only have nonzero entries on the diagonal blocks, thus $\delta_{\gC}$ preserves the block-diagonal structure in the target space. Once the operator is defined as the sum of positive operators applied on a PSD matrix, $\delta_{\gC}$ also preserves the PSD property.

We note that $\operatorname{Eq}(\delta_{\gM}^+, \delta_{\gM}^-) \!=\! \ker \delta_{\gM}$, so equalizers can be used to define $H^0(G,\gM)$ for the sheaf of means. Therefore our construction is a natural extension of the standard case.  Moreover, we can continue extending the theory of cellular sheaves valued on \textbf{Vect} to cellular sheaves valued on \textbf{SDef}. For instance, the following holds:
\begin{theorem}\label{thm:sdefH0}
      We have that $H^0(G,\gC) = \Gamma(G, \gC)$, and for an oriented edge $e = (v,u) \in E$ the output of $\delta_{\gC}$ in $(G,\gC)$ is $\delta_{\gC}(\Sigma)_e = \rmapc{v}{e}(\Sigma_v) + \rmapc{u}{e}(\Sigma_u)$. Furthermore, $\delta_{\gC}(\Sigma) = \sum_i \varphi_{\mathbf{A}_i}\varphi_{\delta_{\gM}}(\Sigma)$.
\end{theorem}
Now, we use $\delta_{\gC}$ to define the Laplacian diffusion operator.

\begin{definition}\label{def:lap_sheaf_covariances}
    The sheaf Laplacian operator for a sheaf $(G,\gC)$ valued on \textbf{SDef} over an undirected graph $G$ is the map $L_{\gC}:C^0(G,\gC) \to C^0(G,\gC)$ defined on a 0-cochain $\Sigma$ as $L_{\gC}(\Sigma) = \sum_{i=1}^n \varphi_{\mathbf{B}_i^{\phantom{+}}}\!\!\!(\varphi_{(\delta_{\gM}^+)^{\top}}(\delta_{\gC}(\Sigma)) + \varphi_{(\delta_{\gM}^-)^{\top}}(\delta_{\gC}(\Sigma)))$, with $\mathbf{B}_i = (\mathbf{e}_i^{\phantom{\top}}\!\mathbf{e}_i^{\top}) \otimes I_d \in \sR^{nd \times nd}$.  
Moreover,
    \begin{equation}
        L_{\gC}(\Sigma)_v = \sum_{v,u \unlhd e} \varphi_{\rmapm{v}{e}^{\top}}(\varphi_{\rmapm{v}{e}^{\phantom{\top}}}(\Sigma_v) + \varphi_{\rmapm{u}{e}^{\phantom{\top}}}(\Sigma_u)).
    \end{equation}\label{eq:sheaf_lap}
\end{definition}

In words, the covariance Laplacian at node $v$ aggregates pushforwards of neighboring covariances along the restriction maps, mirroring the neighborhood aggregation of the standard sheaf Laplacian $L_\gM$ but acting on the PSD cone rather than on a vector space.

We can also describe this operator in terms of $L_{\gM}$. Let $k = \max_{v \in V} |N(v)|$ and decompose $L_{\gM} = \sum_{i=1}^{k} L_i + L'$, where each $L_i$ is the block-diagonal matrix with $(L_i)_{vv} = \rmapm{v}{e_i}^{\top}\rmapm{v}{e_i}$ if $v$ has an $i$-th neighbor (and zero otherwise), and $L'$ contains the off-diagonal entries $(L')_{vu} = -\rmapm{v}{e}^{\top}\rmapm{u}{e}$. The block-diagonals satisfy $(\sum_i L_i)_{vv} = (L_{\gM})_{vv}$, so $L_{\gC}(\Sigma) = \sum_{i=1}^{k} \varphi_{L_i}(\Sigma) + \sum_{j=1}^n \varphi_{\mathbf{B}_j}(\varphi_{L'}(\Sigma))$.

We also define the  normalized covariance sheaf Laplacian:
\begin{equation}
\Delta_{\gC}(\Sigma) = \sum_{i=1}^{k} \varphi_{\Delta_i}(\Sigma) + \!\sum_{j=1}^n \varphi_{\mathbf{B}_j\Delta'}(\Sigma),    
\end{equation}
with $\Delta_i = D^{-\frac{1}{2}}L_iD^{-\frac{1}{2}}$ and $\Delta' = D^{-\frac{1}{2}}L'D^{-\frac{1}{2}}$.

\cref{thm:sdefLap} shows that the identity $\ker L_{\gM} = H^0(G,\gM)$ in the \textbf{Vect}-valued sheaf is translated to the \textbf{SDef}-valued sheaf in terms of equalizers. Concretely, $L_{\gC}$ admits a decomposition $L_{\gC} = L_{\gC}^+ + L_{\gC}^-$ into positive and negative 
parts, mirroring the analogous decomposition of $L_{\gM}$ in \textbf{Vect} (see \cref{sec:app:p}).\looseness=-1

\begin{theorem}
\label{thm:sdefLap}
    Let $G$ be a connected graph and $(G, \gC)$ be a sheaf valued in \textbf{SDef}. Then $\operatorname{Eq}(L_{\gC}^+, L_{\gC}^-) = H^0(G,\gC)$.
\end{theorem}

%
\vspace{-5pt}
\begin{mybox}{}
\vspace{-5pt}
\textbf{In summary}, The sheaf of covariances is a \textbf{SDef}-valued cellular sheaf, whose respective notions of coboundary map and sheaf Laplacian are induced by \textbf{Vect}-valued cellular sheaf valued. Moreover, central results regarding the standard sheaf Laplacian are extended to a sheaf Laplacian in \textbf{SDef}.\looseness=-1
\vspace{-5pt}
\end{mybox}

\subsection{Gaussian sheaves}\label{subsection:means+covariances}

We now define Gaussian sheaves combining sheaves of means (in \textbf{Vect}) and of covariances (in \textbf{SDef}). \looseness=-1 
\begin{definition}\label{def:gauss_sheaf}
    Given an undirected graph $G=(V,E)$, a Gaussian (cellular) sheaf $(G, \gF)$ associates: 
     \vspace{-4pt}
\begin{enumerate}[leftmargin=16pt, itemsep=1pt, nosep]
        \item  A product space $\gF(v) \coloneqq (\sR^{d_v}, \psd^{d_v})$ for every $v \in V$.
        \item  A product space $\gF(e) \coloneqq (\sR^{d_e}, \psd^{d_e})$ for every $e \in E$.
        \item  A map $\Phi_{\rmapf{v}{e}}\!\!:\! \gF(v) \!\!\to\!\! \gF(e)$ for every incident pair $v \unlhd e$, w/ $\Phi_{\rmapf{v}{e}}(\mu, \Sigma) =  (\rmapf{v}{e}\mu,\varphi_{\rmapf{v}{e}}(\Sigma))$. \looseness=-1
    \end{enumerate}
    \vspace{-5pt}
\end{definition}
Note that $H^0(G, \gF) \!=\! H^0(G, \gM) \!\times\! H^0(G, \gC)$.
We denote by $(G,\gF)$ the Gaussian sheaf over $G$ and from now on this will be our main object of interest. We fix a common dimension $d_v = d_e = d$ for all stalks in our sheaf. \looseness=-1

These constructions can be described as a sheaf valued in a category of Gaussian distributions, since there is a bijection $\gauss{d} \simeq \sR^d \times \psd^d$, where $\gauss{d}$ denotes the space of Gaussian distributions on $\sR^d$. Both descriptions are equivalent --- see Appendix \ref{sec:app:gauss_sheaf} for details.
The sheaf Laplacian for the Gaussian sheaf is obtained by the combination of the means and covariances' Laplacian.
\begin{definition}\label{def:lap_sheaf_gaussian}
    For every $(\mu,\Sigma) \in C^0(G,\gF)$, the Gaussian sheaf Laplacian is the map $L_{\gF}\!:\! C^0(G,\gF) \!\to \!C^0(G,\gF)$  given by \looseness-1
    \begin{equation*}
        L_{\gF}(\mu,\Sigma) = (L_{\gM}\mu,L_{\gC}(\Sigma)).
    \end{equation*}
\end{definition}

Similarly, we define the normalized Gaussian sheaf Laplacian as $\Delta_{\gF}(\mu,\Sigma) = (\Delta_{\gM}\mu, \Delta_{\gC}(\Sigma))$.

\begin{example}
    Consider the constant sheaf on $G$, i.e., a sheaf whose stalks are $(\sR, \sR_+)$ and whose restriction maps $\rmapm{v}{e}$ are identities (providing $\rmapc{v}{e}$ are identities too). On the mean side, the Laplacian $L_{\gM}$ is the graph Laplacian $L_G$. On the covariance side we obtain $L_{\gC}(\Sigma)_v = \deg(v)\Sigma_v + \sum_{u \in N(v)} \Sigma_u$ or, equivalently, $L_{\gC}(\Sigma) = \sum_j \mathbf{B}_jL_{G}\Sigma L_G^{\top}\mathbf{B}_j^{\top}$. 
\end{example}
The Laplacian for Gaussian sheaves generalizes the graph Laplacian in the sense that the constant sheaf valued on $\textbf{Vect} \times \textbf{SDef}$ can be made into the constant sheaf valued on \textbf{Vect} by letting all stalks be $(\sR,\{0\})$, where it holds that $\textbf{Vect} \times \mathbf{0}$ is isomorphic to $\textbf{Vect}$, and $\mathbf{0}$ denotes the trivial category with only one object. We could also retrieve the cellular sheaf of \citet{bodnar2022neural} if we let the stalk spaces be $(\sR^d, \{0\})$. Notice that both last examples would describe Dirac delta distributions. \looseness=-1

\section{Analyzing classes of restriction maps}\label{subsec:diff_restmaps}
Since $\mathcal{G}(\mathbb{R}^d)$ is closed under linear maps, we can describe their action on Gaussians as group actions $\mathfrak{g} \acts \mathcal{G}(\mathbb{R}^d)$, with $\mathfrak{g}$ the group containing the restriction maps. Following \citet{bodnar2022neural}, we consider $\mathfrak{g} \in \{GL(d), O(d), D(d)\}$ and analyze the consequences of each choice.

In what follows, we fix a Gaussian $\vartheta \in \mathcal{G}(\mathbb{R}^d)$ with parameters $(\mu, \Sigma)$ and define, for a path $\gamma = (v, v_1, \dots, v_{\ell}, u)$, the transport $\mathbf{P}^\gamma_{v \to u}: \mathcal{F}(v) \to \mathcal{F}(u)$ as the composition $\Phi_{\rmapf{u}{e_1}^{\top}} \circ \Phi_{\rmapf{v_{\ell}}{e_1}} \circ \cdots \circ \Phi_{\rmapf{v_1}{e_{\ell+1}}^{\top}} \circ \Phi_{\rmapf{v}{e_{\ell+1}}}$ along the edges of $\gamma$.

\begin{proposition}\label{prop:orbitorth}
    A Gaussian $\varrho \in \mathcal{G}(\mathbb{R}^d)$ with parameters $(\mu', \Sigma')$ is in the orbit of $\vartheta$ under the action of $O(d)$ only if $\|\mu'\|_2 = \|\mu\|_2$ and $\Sigma'$ has the same eigenvalues as $\Sigma$.
\end{proposition}

Hence $O(d)$ maps transport $\nu_v$ only to Gaussians whose mean lies on the sphere of radius $\|\mu_v\|$ and whose covariance is similar to $\Sigma_v$. Global sections then exist when $\|\mu_v\| = c$ and $\Sigma_v \in [\Sigma^{\ast}]$ for all $v$, with some $c \in \sR$ and $\Sigma^{\ast} \in \psd^d$.

\begin{proposition}\label{prop:staborth}
    Let $\lambda_1, \dots, \lambda_k$ be the eigenvalues of $\Sigma$ with multiplicities $m_1, \dots, m_k$. If we write $\mu = (\mu_1, \dots, \mu_k)$ where $\mu_j \in \sR^{m_j}$, the stabilizer of $\vartheta$ is in bijection with $\prod_{j=1}^k O(m_j-\mathbbm{1}_{\{\mu_j \neq 0\}})$. If all eigenvalues are different, the stabilizer is in bijection with $(\sZ_2)^{|\{j \:|\: \mu_j = 0\}|}$.
\end{proposition}

Therefore, a transported $\nu_v$ with distinct covariance eigenvalues has a strict stabilizer and will in general differ from its transport into $\gF(u)$. 
\cref{prop:section_existence} gives us a clean characterization of when global sections exist.


\begin{proposition}\label{prop:section_existence}
    Let $(G,\gF)$ be a Gaussian sheaf with $O(d)$ maps over a connected graph $G$. Then $\exists \nu \in H^0(G,\gF)$ with $\nu_v = \vartheta$ if, and only if, for every cycle $\gamma$ based at $v$ holds that $\mathbf{P}^{\gamma}_{v \to v} \in \operatorname{Stab}(\vartheta)$.
\end{proposition}

When stabilizers are trivial, the only way for $\mathbf{P}^{\gamma}_{v \to v}$ to fix $\nu_v$ is for it to be the identity, i.e.\ $\rmapf{v}{e_{\ell}}^{\top}\rmapf{v_{\ell}}{e_{\ell}} \cdots \rmapf{v_1}{e_{0}}^{\top}\rmapf{v}{e_{0}} = I_d$. A single node with trivially stabilized distribution thus imposes a stringent condition on every cycle through it. When the transport is path-independent, this product is automatically the identity and global sections always exist; we now characterize them in this regime.\looseness=-1

\begin{proposition}\label{prop:section_bijection}
    Let $(G,\gF)$ be a Gaussian sheaf with $O(d)$ maps over a connected graph $G$ where the transport is path-independent. Then $H^0(G,\gF)$ and $\gG(\sR^d)$ are in  bijection.
\end{proposition}

\cref{prop:section_bijection} generalizes the result of \citep{bodnar2022neural} that under path-independent transport $H^0(G,\gM) \cong \sR^d$.

\begin{proposition}\label{prop:orbitdiag}
    The orbit of $\vartheta$ under the action of $D(d)$ is isomorphic to $D(d)$ and the stabilizer is always trivial, if $\mu_i \neq 0$ for all $i = 1,\dots, n$.
\end{proposition}

This, for each $\mu'$ there is a unique $\varrho \in D(d) \cdot \vartheta$ with mean $\mu'$, whose covariance is determined by $\mu, \mu'$, and $\Sigma$. Restriction maps in $GL(d)$ reach a wider class of Gaussians at the cost of more parameters. \looseness=-1

\begin{proposition}\label{prop:orbitgen}
    Under the action of $GL(d)$ we can achieve any mean vector by restricting the covariances to a specific set. Moreover, if $\Sigma$ is positive-definite, by restricting the means to a specific set then we can also achieve any positive-definite covariance matrix. 
\end{proposition}

\section{Gaussian Sheaf Neural Networks}\label{sec:GSNN}

We have established the foundations of Gaussian sheaves and their Laplacian operators.
Building upon this, we now introduce Gaussian Sheaf Neural Network (GSNN), the first method for DoD regression on networked data. 
Towards this objective, let $G = (V, E)$ be a graph with nodes $V = \{v_{1}, \dots, v_{n}\}$, and assume each $v \in V$ is equipped with a $s$-variate Gaussian distribution $\nu_v = \mathcal{N}_{s}(\mu_{v}, \Sigma_{v})$. 
Also, let $I$ be the $n$-dimensional identity matrix, and let $h$ and $d$ be the hidden and stalk dimensions.
In practice, $\mathcal{N}_{s}(\mu_{v}, \Sigma_{v})$ is estimated from a node-wise empirical observations via maximum likelihood. 

As described next, GSNN takes as input a graph $G = (V, E)$ with distributional features $\{N_{s}(\mu_{v}, \Sigma_{v}) \colon v \in V\}$ and learns to output samples from a (possibly non-Gaussian) distribution. 
\looseness=-1 
\vspace{-3pt}
\begin{mybox}{}
\vspace{-6pt}


\textbf{Gaussian Sheaf Neural Networks.} The $\ell$-layer GSNN with weights $W_{o} \in \mathbb{R}^{h \times d \times s}$, $W_{1} \in \mathbb{R}^{d \times d}$, and $W_{2} \in \mathbb{R}^{h \times h}$ is defined by the following operations. \looseness=-1 
\begin{enumerate}[leftmargin=20pt]
    \item Let $\oplus$ be the concatenation operator. We initially embed each $\nu_{s}$ into the sheaf stalks, 
    \begin{equation*}
        \mu_0 \gets \bigoplus_{i \in [n]} W_o \mu_{v_{i}}, \qquad \Sigma_0 \gets \bigoplus_{i\in [n]} W_{o} \Sigma_{v_{i}} W_{o}^{\top},
        \end{equation*}
    where we broadcast tensor-matrix products fixing the first dimension of the tensor. We then reshape into $\mu \in \mathbb{R}^{nd \times h}$ and $\Sigma \in \mathbb{R}^{nd \times nd \times h}$.  
    Intuitively, each node $v_{i}$ carries $h$ copies of a $d$-dimensional Gaussian distribution.
    \item We apply the weights on both parameters, broadcasting the application of $\varphi_{I_n \otimes W_1}$ on $\Sigma$ over its last dimension $h$:
    \begin{equation*}
        \mu_0 \gets (I_n \otimes W_1)\mu_0W_2, \qquad \Sigma_0 \gets (I_n \otimes W_1)\Sigma_0(I_n \otimes W_1)^{\top}W_2
    \end{equation*}
    \item Let $\Delta_{\gM}$ and $\Delta_{\gC}$ be the normalized Laplacians of the sheaf of means $\gM$ and of covariances $\gC$, respectively. We compute \looseness=-1
    \begin{equation*}
        \mu_{\ell} = (I - \Delta_{\mathcal{M}})^{\ell}\mu_0, \qquad \Sigma_{\ell} = (I + \Delta_{\gC})^{\ell}\Sigma_0
    \end{equation*}
    \item  As a final step, we obtain samples from the approximate target distribution for each node $v \in V$ by drawing $T$ samples from the learned distribution $\mathcal{N}(\mu^\prime_v, \Sigma^\prime_v)$, $(\nu_v^\prime)_{0}, \ldots, (\nu_v^\prime)_{T}$, for each node $v$, and pushing them through an MLP.
    \vspace{-6pt}
\end{enumerate}





\end{mybox}
GSNN takes the $\ell$-th power of the Laplacian rather than interleaving $\ell$ 1-hop iterations with non-linearities, in the spirit of simple graph convolution~\cite{wu2019simplifying}. This is computationally cheaper and, importantly, preserves positive-semidefiniteness of $\Sigma^\prime$.

Notably, we train our model using the 2-Wasserstein distance as a loss function. More specifically, we compute the 2-Wasserstein distance between GSNN output samples and those observed for the nodes in the training split, using the Sinkhorn algorithm \citep{bonneel2011displacement}. 
We further discuss the merits of this choice in \cref{sec:app:model}.
 


Additionally, restriction maps are parametrized for each incident pair $v \unlhd e \coloneqq (v,u)$ as
\vspace{-2pt}
\begin{equation}\label{eq:gauss_rmap}
    F_{v \unlhd e\coloneqq(v,u)} = \sigma(\Psi(\concat{[\concat{\mu_{v}}{\mu_{u}}]}{[\concat{\Sigma_{v}}{\Sigma_{u}}]})),
\end{equation}
\vspace{-2pt}
in which $\sigma$ is a non-linear activation function and $\Psi$ an MLP. 
Importantly, \citet{bodnar2022neural} show MLPs can learn any sheaf $(G, \mathcal{F})$ whose stalks are vector spaces given $(\mathbf{x}_v, \mathbf{x}_u) \!\neq\! (\mathbf{x}_w, \mathbf{x}_z)$ for all $(v,u) \!\neq\! (w,z) \!\in\! E$. 
Since the covariance sheaf restriction maps are determined by those of the mean sheaf, $\Psi$ may learn any Gaussian sheaf (recall \Cref{def:lap_sheaf_covariances}).


\section{Experiments} \label{sec:experiments}
In this section, we provide an empirical illustration of GSNN's behavior, showing that it produces lower 2-Wasserstein distance to the target distribution than vector-based baselines on a small set of synthetic and real-world datasets. In \cref{subsec:ablation}, we introduce two additional baselines to highlight the importance of aligning model design with the underlying data structure, as well as the benefits of learning the sheaf Laplacian. Overall, we position these experiments as empirical validations that complement our theoretical contributions,  than a comprehensive empirical study. 

\subsection{Distribution-on-distribution regression over graphs}

\noindent \textbf{Simulated data.} We simulate four datasets, each consisting of a graph with 200 nodes linked using  the Barabasi-Albert or the Watts-Strogatz models. For the Barabasi-Albert model, we set the attachment parameter $m\!\in\! \{25, 50\}$. For the Watts-Strogatz model, mean degree $k\! \in \!\{25, 45\}$ and rewiring probability $p=\{0.3, 0.5\}$.
To define node $v$'s distributions, we draw a mean vector $\mu_{\nu_v} \sim \mathcal{N}(U_1, U_2U_2^T)$ and a covariance matrix $\Sigma_{\nu_v} \sim \textit{Wishart-Inv}(U_2U_2^T)$, where $U_1 \in \mathbb{R}^s$ and $U_2 \in \mathbb{R}^{s \times s}$ are drawn point-wise from a uniform with support $[-1,1]$.
Denoting by $N(v)$ the $1$-hop of vertex $v$ and $\nu_v$ its associated Gaussian distribution, we create the target distribution $y_v$ as $\nu_v \ast \Conv_{u \in N(v)} \alpha_u \nu_u$, where $*, \Conv$ denote the convolution operation. The constants $\alpha_u \! = \! \nicefrac{D_{KL}(\mathbf{x}_v||\mathbf{x}_u)^{-1}\!}{\!\max\limits_{u \in N(v)} \! D_{KL}(\mathbf{x}_v||\mathbf{x}_u)^{-1}},\forall u \!\in \! N(v)$ weigh each neighbor's  contribution to the output distribution. \looseness=-1

\noindent \textbf{Real-world data.} We consider two real-world datasets consisting of daily weather data from cities in Canada from \href{https://www.weatherstats.ca}{Canada Weather Stats},  with 200/12524 and 201/12630 nodes/edges. The neighbors of a given city are those cities that lie within a circle of radius of $200$km centered at it. For the first dataset we use the joint distribution of average temperature and dew point for each day in the spring of 2018 to predict the average humidity and pressure station for each day in the summer of 2018. In the second dataset, we use the average dew point and pressure station to predict the average temperature and humidity, for the same period. For both datasets, we recover the mean vector and covariance matrix from the daily measurements to use as input for GSNN, and use the raw measurements as input for the vector-based neural networks.

\begin{table*}[b]
    \centering
    \caption{Mean and standard deviation of the 2-Wasserstein distance for each model and dataset. The two best performances are in {\bf \color{ForestGreen}{green}} and {\bf\color{DarkOrchid}{purple}} respectively. Our model (\textbf{GSNN}) shows the best performance in 5 out of 6 datasets.} 
    \vspace{-1pt}
\begin{adjustbox}{width=\textwidth,center}
\begin{tabular}{ccccccc}
    \toprule
      {} & Barabasi-Albert & Barabasi-Albert & Watts-Strogatz & Watts-Strogatz & Weather1 & Weather2 \\
    Parameters & m = 25 & m = 50 & k = 25, p = 0.3 & k = 45, p = 0.5 & - & - \\
    \midrule
    MLP & $\values{10.04}{0.39}$ & $\values{12.85}{0.63}$ & $\values{13.37}{1.39}$ & $\values{13.37}{1.39}$ & $\values{7.80}{0.57}$ & $\valuess{3.23}{0.21}$  \\
    GCN & $\values{10.12}{0.89}$ & $\values{13.90}{0.34}$ & $\values{13.41}{0.42}$ & $\values{13.41}{0.42}$ & $\values{7.40}{0.77}$ & $\values{4.74}{0.25}$ \\
    \midrule
    Diag-NSD & $\values{9.86}{1.36}$ & $\values{9.86}{1.36}$ & $\values{11.27}{1.00}$ & $\valuesf{10.85}{0.47}$ & $\values{5.79}{0.29}$ & $\values{5.05}{2.06}$ \\
    O(d)-NSD & $\values{9.91}{1.27}$ & $\values{11.12}{0.55}$ & $\valuess{9.87}{1.27}$ & $\values{11.91}{1.20}$ & $\values{5.83}{1.04}$ & $\values{4.43}{1.52}$\\
    Gen-NSD & $\values{9.63}{1.79}$ & $\values{9.63}{1.79}$ & $\values{11.52}{0.89}$ & $\values{11.46}{0.87}$ & $\values{5.86}{0.59}$ & $\values{6.38}{4.46}$ \\
    \midrule
    \textbf{Diag-GSNN} & $\values{8.66}{0.72}$ & $\values{8.76}{0.56}$ & $\values{12.22}{0.55}$ & $\values{12.22}{0.55}$ & $\valuess{5.46}{0.11}$ & $\values{3.92}{0.19}$ \\
    \textbf{O(d)-GSNN} & $\valuesf{8.45}{0.61}$ & $\valuess{8.63}{0.62}$ & $\values{12.29}{1.29}$ & $\values{11.80}{0.70}$ & $\valuesf{5.41}{1.03}$ & $\values{3.41}{0.41}$ \\
    \textbf{Gen-GSNN} & $\valuess{8.48}{0.47}$ & $\valuesf{8.48}{0.47}$ & $\valuesf{8.87}{0.42}$ & $\valuess{10.87}{0.42}$ & $\values{5.55}{0.78}$ & $\valuesf{2.64}{0.30}$ \\
    \bottomrule
    \end{tabular}
\end{adjustbox}
  \label{tab:results}
   \vspace{-12pt}
\end{table*}

\noindent \textbf{Baselines.} To our best knowledge, GSNN is the first method for distribution-on-distribution regression over graphs. Therefore, we propose five baselines to compare GSNN against. The core idea of the baselines is to sample values from the node distributions and push them through well-established predictive models, along with the graph skeleton if the model permits. More specifically, we consider MLPs, GCNs, and NSD with different types of restriction maps. Similar to our model, all baselines are trained using the 2-Wasserstein distance to the target node distributions as loss. \looseness=-1

\noindent \textbf{Experimental setup.} We use a $60\%/20\%/20\%$ node split for train, test, and validation. All experiments were repeated ten times with different seeds. We provide more details on hyperparameter search in \cref{tab:hyp}, and the runtime is available at \cref{tab:runtime}.

\noindent \textbf{Results.} \Cref{tab:results} reports the results for GSNN and all baselines in terms of average 2-Wasserstein distance and standard deviation. Notably, GSNN achieves the best performance on five out of six datasets. As expected from \cref{prop:orbitgen}, general restriction maps (Gen-GSNN) typically lead to better results, yielding the best or second-best performance in 4/6 datasets. There is a statistical relevance analysis in \cref{app:stat_relevance} to confirm that our results are statistically better than other models.

\subsection{Further assessment of GSNN's geometry-informed convolutions}\label{subsec:ablation}

To further highlight the importance of aligning model design with the underlying data structure, we compare GSNN against two additional baselines, which are described below.

\textbf{GaussianGCN.} To test GSNN against a naive approach for graphs with Gaussian distributions as vectors, we implement GaussianGCN as a GCN receiving concatenated Gaussian parameters for each node, i.e. a 5-dimensional vector $\mu || vech(\Sigma)$, where $vech$ is the vectorized lower triangular matrix of $\Sigma$. The GCN output is a mean vector and 3 numbers representing the covariance for each node, which may not form a PSD matrix. We use a linear layer to map them to two scalars, and then create a diagonal covariance matrix whose entries are the exponential of those values.

\textbf{GSNN-GraphLap.} To verify the necessity of learning restriction maps to generate the corresponding sheaf Laplacian, we implement  GSNN-GraphLap as a version of GSNN where we use the graph Laplacian.\looseness=-1 

\textbf{Results.} The results in \cref{tab:adapted_graphLap} show that  GSNN significantly surpasses both  the GaussianGCN and GSNN-GraphLap. This indicates that learning the sheaf through the restrictions maps is beneficial and that models developed to respect the natural structure of data (in this case, Gaussian distributions) are necessary. We hope this work encourages new distribution-based regression benchmarks.  

\begin{wraptable}[13]{r}{0.55\textwidth} 
\begin{minipage}{\linewidth}
    \centering
    \scriptsize
\scriptsize
\setlength{\tabcolsep}{3.5pt}
\vspace{-15pt}
\caption{Mean and standard deviation of the 2-Wasserstein distance on selected datasets. GSNN consistently outperforms GSNN with standard graph Laplacian -- showing that learning the Laplacian with sheaves is beneficial -- and GCN receiving concatenated Gaussian parameters -- pointing to the value of respecting the intrinsic structure of data.}
\vspace{2pt}
\label{tab:adapted_graphLap}
\begin{tabular}{lccc}
\toprule
Model & BA ($m=50$) & WS ($k=45, p=0.5$) & Weather1 \\
\midrule
GaussianGCN & 14.83 $\pm$ 0.33 & 14.41 $\pm$ 0.34 & 6.15 $\pm$ 0.89 \\ 
GSNN-GraphLap & 11.47 $\pm$ 0.29  & \textbf{{\color{DarkOrchid}11.48 $\pm$ 0.42}}  & 8.64 $\pm$ 1.25 \\
\midrule
\textbf{Diag-GSNN} & 8.76 $\pm$ 0.56  & 12.22 $\pm$ 0.55  & \textbf{{\color{DarkOrchid}5.46 $\pm$ 0.11}} \\
\textbf{O(d)-GSNN} & \textbf{{\color{DarkOrchid}8.63 $\pm$ 0.62}}  & 11.80 $\pm$ 0.70  & \textbf{{\color{ForestGreen}5.41 $\pm$ 1.03}} \\
\textbf{Gen-GSNN}  & \textbf{{\color{ForestGreen} 8.48 $\pm$ 0.47}} & \textbf{{\color{ForestGreen}10.87 $\pm$ 0.42}} & 5.55 $\pm$ 0.78 \\
\bottomrule
\end{tabular}
\vspace{-7pt}
\end{minipage}
\end{wraptable}

\section{Related Work}
\textbf{Sheaf Neural Networks.} SNN models were introduced in \citet{hansen2020sheaf}, initiating the interest of the machine learning community in cellular sheaves whose stalks are vector spaces \citet{barbero2022sheaf, bodnar2022neural, NEURIPS2023sheafhypergraph}. More recently \citet{gillespie2024bayesiansheafneuralnetworks} using Bayesian inference to learn sheaves. The cellular sheaves in these works are all valued in \textbf{Vect}; consequently, they do not exploit the geometric and algebraic structure that arises when node features are probability distributions, which is the inductive bias targeted by GSNN.  \looseness=-1

\noindent\textbf{Distribution-on-distribution regression.} While this is the first work dealing with networked data, the literature on distribution-on-distribution regression for independent data is blooming. For instance,
\citet{ghodrati2022distribution} perform regression when both input and output are probability distributions on a compact interval.
\citet{chen2023wasserstein} develop a linear distribution-on-distribution model for univariate distributions.
\citet{zhang2022wasserstein} proposes an autoregressive model for density time series by leveraging the  tangent space structure on the
space of distributions induced by the Wasserstein metric.
\citet{okano2024distribution} consider the case in which both input and target are Gaussian distributions, using a nearly isometric map from $\gauss{d}$ to a vector space and then using a linear model to perform the regression of one Gaussian into another.

\noindent\textbf{Cellular Sheaves.} \citet{ghrist2022cellular} and \citet{riess2022diffusion} lay the theoretical foundations for sheaves whose stalks are partially ordered sets (i.e., lattices). Although lattices do not intuitively resemble distributions, they form a category that imposes similar technical difficulties. For instance, both develop a novel formulation for the Laplacian, and $H^0(G,\gF)$ did not coincide with the kernel of the former. In fact, we defined our Lyapunov energy inspired by their work.   \looseness=-1

\section{Discussion}\label{sec:discussion} 

This work proposed Gaussian Sheaf Neural Networks (GSNN), the first method for distribution-on-distribution regression over graphs. Towards that end, we extended the theory of cellular sheaves to define the  Gaussian sheaf and its Laplacian. 
Since in real-world applications there is no recipe to define restriction maps \emph{a priori}, we propose learning the Gaussian sheaf in an end-to-end fashion.

After defining a transport between nodes through the composition of restriction maps, we investigated how its expressivity along a path depends on whether the restriction maps are in $O(n)$, $D(n)$, or $GL(n)$ (\cref{subsec:diff_restmaps}). \Cref{subsec:expressive} (in the appendix) further studies oversmoothing in our setting, showing that GSNN can control a Lyapunov energy through its parameters. We also validated GSNN with experiments on four simulated and two real-world datasets.

The main limitation of GSNN is that we assume the input node distributions can be encoded as multivariate Gaussians, even though GSNN's output is a free-form distribution. We believe extending GSNN to different input distributions is a promising direction for future works.

Beyond GSNN itself, our work shows how category-theoretic principles can extend sheaf-based deep learning to node features beyond vector spaces, reinforcing the view that category theory provides a unifying framework for architecture design \cite{gavranovicposition}. \looseness-1

\textbf{Societal Impact.} We foresee no immediate ethical or societal risks associated with the developments presented in this paper.

\bibliographystyle{plainnat}
\bibliography{sample}


\appendix

\section{Oversmoothing: Gaussian sheaves and the Lyapunov energy}\label{subsec:expressive}

Oversmoothing in GNNs occurs when increasing network depth causes node embeddings within a connected component to become indistinguishable. This phenomenon can be assessed using the Dirichlet energy, which quantifies the smoothness of node embeddings by summing the squared differences between connected nodes \cite{CaiWang2020OverSmoothing}. The message-passing mechanism in deep GNNs behaves similarly to iterative Laplacian smoothing, effectively minimizing the Dirichlet energy. As layers deepen, this energy decreases, driving embeddings toward uniformity and reducing their discriminative power, ultimately degrading predictive performance \cite{oono2019graph}.\looseness=-1

Although not explicitly related to the NSD model's capacity to avoid oversmoothing, \citet{bodnar2022neural} proves that sheaf diffusion models in \textbf{Vect} can both increase or decrease the Dirichlet energy of the sheaf, given by $ E_{\gF}(\mathbf{x}) \coloneqq \mathbf{x}^{\top} \Delta_{\gF}\mathbf{x}$, in circumstances that Graph Convolution Networks (GCNs) cannot. This indicates that sheaves provide flexible models by default. Here we show that GSNN shares this same malleable behavior.

A simple verification reveals $\mathbf{x} \in\! \ker(\Delta_{\gF}) \Leftrightarrow E_{\gF}(\mathbf{x})=0$. However, for \textbf{SDef}-valued sheaves,  the kernel of $\Delta_{\gC}$ is always trivial, indicating that the Dirichlet energy a poor measure in our context. Therefore, to analyze GSNN, we propose the Lyapunov energy.

\begin{definition}
    Let $d$ be a distance function between probability measures. The Lyapunov energy of a Gaussian sheaf is given by \looseness=-1
    \vspace{-5pt}
    \begin{equation*}
        V(\nu) = \sum_{(v,u) \in E} d\left(\Phi_{\rmapf{v}{e}D_v^{-\frac{1}{2}}}(\nu_v), \Phi_{\rmapf{u}{e}D_u^{-\frac{1}{2}}}(\nu_u)\right)^2
    \end{equation*}
    \vspace{-12pt}
\end{definition}
Using the 2-Wasserstein, and writing $\Sigma_v' = \varphi_{\rmapf{v}{e}D_v^{-\frac{1}{2}}}(\Sigma_v)$, the Lyapunov energy reduces to\looseness=-1
\vspace{-5pt}
\begin{equation*}
    V(\nu) \!=\! E_{\gM}(\mu) + \!\!\!\sum_{(v,u) \in E} \!\!\!\operatorname{tr}\!\left(\Sigma_v' + \Sigma_u' - 2\left(\Sigma_v'^{\frac{1}{2}}\Sigma_u'^{\phantom{\frac{1}{2}}}\!\Sigma_v'^{\frac{1}{2}}\right)^{\frac{1}{2}}\right).
\end{equation*}

Note that $V(\nu) = 0$ exactly when $\nu \in H^0(G,\gF)$, and when restriction maps are identities this forces $\nu_u = \nu_v$ across all neighbors — precisely the oversmoothing condition. Thus $V$ is a meaningful measure of `smoothness'.

\Cref{prop:lyapenergy} below shows that GSNN can freely control the Lyapunov energy depending on parameter choice, mirroring the Dirichlet energy behavior of vector-valued sheaf diffusion (recall the pushforward parameterization in \cref{rmk:pushforward}). \looseness=-1

\begin{proposition}
\label{prop:lyapenergy}
    Let $G$ be a connected graph and $(G,\gF)$ a Gaussian sheaf with orthogonal restriction maps. Then there exist a family $\{\nu_j\}_{j \in J} \subset C^0(G,\gF)$ and weight matrices $W_{\alpha}$, $W_{\beta}$ such that $V((I \otimes W_{\alpha})\#\nu_j) > V(\nu_j)$ and $V((I \otimes W_{\beta})\#\nu_j) < V(\nu_j)$.
\end{proposition}

Next, to empirically illustrate that GSNN avoids oversmoothing, we sample a synthetic dataset with 100 nodes using the Barabasi-Albert model with $m=25$. We fix the hyperparameters for all models, since the motivation behind this experiment is illustrate the capacity of GSNN to avoid oversmoothing. Both input and target distributions are bidimensional Gaussians, with 30 samples for the target distributions. \cref{tab:oversm}
shows that, contrary to GCN, GSNN models keeps the 2-Wasserstein distance stable through the layers. In \cref{tab:energy} we provide the Dirichlet energy of GCN and the energy of the sheaf of means of O(d)-GSNN. This results combined suggest that the GSNN ability to control the Lyapounov energy --- foreseen by \cref{prop:lyapenergy} --- is responsible for its general better performance. 

\begin{table*}[ht!]
    \centering
    \caption{Effect of adding more layers to the models.  Both input and target distributions are bidimensional Gaussians, with 30 samples for the target distributions. }
    \vspace{2mm}
\begin{tabular}{ccccc}
    \toprule
      {} & 1 & 2 & 4 & 8 \\
    \midrule
    GCN & $\mathbf{\values{9.67}{0.49}}$ & $\mathbf{\values{10.56}{0.67}}$ & $\values{57.03}{40.91}$ & $\values{30874}{42864}$ \\
    \midrule
    \textbf{Diag-GSNN} & $\values{11.02}{0.91}$ & $\values{11.58}{1.81}$ & $\mathbf{\values{10.55}{1.34}}$ & $\mathbf{\values{10.79}{1.77}}$  \\
    \textbf{O(d)-GSNN} & $\values{10.68}{1.03}$ & $\values{11.79}{1.51}$ & $\values{11.21}{1.57}$ & $\values{11.79}{2.27}$  \\
    \textbf{Gen-GSNN} & $\values{11.68}{1.35}$ & $\values{12.71}{1.95}$ & $\values{11.00}{1.21}$ & $\values{11.75}{2.38}$  \\
    \bottomrule
    \end{tabular}
    \label{tab:oversm}
\end{table*}

\begin{table*}[ht!]
    \centering
    \caption{Dirichlet energy for GCN and sheaf of means of GSNN. We can see that GCN quickly converges to 0 energy, while the energy for GSNN is stable.}
    \label{tab:A}
    \vspace{2mm}
\begin{adjustbox}{width=0.95\textwidth,center}
\begin{tabular}{ccccccccccc}
    \toprule
      {} & 1 & 2 & 3 & 4 & 5 & 6 & 7 & 8 \\
    \midrule
    GCN & $23495.02$ & $99.85$ & $0.79$ & $0.04$ & $0.0033$ & $0.0009$ & $0.000035$ & $0.000007$  \\
    \midrule
    \textbf{O(d)-GSNN} & $66.7167$ & $66.7315$ & $66.7367$ & $66.7383$ & $66.7387$ & $66.7388$ & $66.7388$ & $66.7388$ \\
    \bottomrule
    \end{tabular}
\end{adjustbox}
    \label{tab:energy}
\end{table*}

\section{Proofs} \label{sec:app:p} 
\subsection{Proof of \textbf{\texorpdfstring{\cref{thm:sdefH0}}{Theorem}}}

\textbf{\cref{thm:sdefH0}. }      We have that $H^0(G,\gC) = \Gamma(G, \gC)$, and for an oriented edge $e = (v,u) \in E$ the output of $\delta_{\gC}$ in $(G,\gC)$ is $\delta_{\gC}(\Sigma)_e = \rmapc{v}{e}(\Sigma_v) + \rmapc{u}{e}(\Sigma_u)$. Furthermore, $\delta_{\gC}(\Sigma) = \sum_i \varphi_{\mathbf{A}_i}\varphi_{\delta_{\gM}}(\Sigma)$. 

\begin{proof}
First we introduce some notation and use an example to illustrate the operations. Let $B \in \sR^{nd \times nd}$ be a matrix defined block-wise by $d \times d$ matrices. We will denote the $i$-th block-line of $B$ as $B_i$ and $B^i$ its $i$-th block-column. For instance if $n=3$, we have
\begin{equation*}
    B = \begin{bmatrix}
        B_{11} & B_{12} & B_{13} \\
        B_{21} & B_{22} & B_{23} \\
        B_{31} & B_{32} & B_{33}
    \end{bmatrix},\:
    B_1 = \begin{bmatrix} B_{11} & B_{12} & B_{13} \end{bmatrix},\:
    B^1 = \begin{bmatrix} B_{11} \\ B_{21} \\ B_{31} \end{bmatrix}.
\end{equation*}

Also, $\mathbf{A}_k = \oplus_{i \in [|E|]} A_{ki}$, with $A_{ki} = 0_{d \times d}$ if $i \neq k$ and $A_{ki} = I_{d \times d}$ if $i = k$. For $|E|=3$ we have
\begin{equation*}
    \mathbf{A}_1 = \begin{bmatrix} 
                    I_{d \times d} & 0 & 0 \\ 
                    0 & 0 & 0 \\ 
                    0 & 0 & 0
                    \end{bmatrix},\:
    \mathbf{A}_2 = \begin{bmatrix} 
                    0 & 0 & 0 \\
                    0 & I_{d \times d} & 0 \\
                    0 & 0 & 0
                    \end{bmatrix},\:
    \mathbf{A}_3 = \begin{bmatrix} 
                    0 & 0 & 0 \\
                    0 & 0 & 0 \\
                    0 & 0 & I_{d \times d}
                    \end{bmatrix}
\end{equation*}

Notice that the operator $\varphi_{\mathbf{A}_k}$ is simply extracting the $k$-th block-diagonal of $\varphi_{\delta_{\gM}^+}(\Sigma)$ and $\varphi_{\delta_{\gM}^-}(\Sigma)$, so we only need to find these entries. We label the lines of $\delta_{\gM}$ by edges, the columns by vertices, and then each line has only two entries  --- since the entries are matrices, these are not exactly the lines and columns, but we can group them that way as we did in the example. Thus the positive and negative parts will have only one non-zero entry at each line. The output at an edge $e \coloneqq (v,u) \in E$ of this operator gives a nice picture of what is happening:
\begin{equation*}
    \delta_{\gM}(\mu)_e = \rmapf{v}{e}\mu_v - \rmapf{u}{e}\mu_u = \delta_{\gM}^+(\mu)_e - \delta_{\gM}^-(\mu)_e
\end{equation*}

Since each line has only one nonzero entry, the structure of $\delta_{\gM}^+\Sigma$ is the same of $\delta_{\gM}^+$, i.e. the nonzero entries are exactly the same. Letting $e_i \coloneqq (v_i,u_i) \in E$:
\begin{equation*}
    \begin{split}
        \varphi_{\delta_{\gM}^+}(\Sigma)_{ii} &= (\delta_{\gM}^+\Sigma)_i \cdot [(\delta_{\gM}^+)^{\top}]^i = \rmapf{v_i}{e_i} \Sigma_{v_i} \rmapf{v_i}{e_i}^{\top} \\
        \varphi_{\delta_{\gM}^-}(\Sigma)_{ii} &= (\delta_{\gM}^-\Sigma)_i \cdot [(\delta_{\gM}^-)^{\top}]^i = \rmapf{u_i}{e_i} \Sigma_{u_i} \rmapf{u_i}{e_i}^{\top}
    \end{split},
\end{equation*}

since the $i$-th column of the transpose is the line of the original matrix, with each entry block transposed. Therefore, if $\delta_{\gC}^+ = \sum_k \varphi_{\mathbf{A}_k}\varphi_{\delta_{\gM}^+}$ and $\delta_{\gC}^- = \sum_k \varphi_{\mathbf{A}_k}\varphi_{\delta_{\gM}^-}$, the output of $\delta_{\gC} = \delta_{\gC}^+ + \delta_{\gC}^-$ on the $i$-th edge/line is
\begin{equation*}
    \delta_{\gC}(\Sigma)_{e_i} = \delta_{\gC}(\Sigma)_i = \rmapf{v_i}{e_i} \Sigma_{v_i} \rmapf{v_i}{e_i}^{\top} + \rmapf{u_i}{e_i} \Sigma_{u_i} \rmapf{u_i}{e_i}^{\top},
\end{equation*}

as expected. Also
\begin{equation*}
    \begin{split}
       H^0(G,\gC) = \operatorname{Eq}(\delta_{\gC}^+, \delta_{\gC}^-) &= \{\Sigma \in C^0(G,\gC) \:|\: \delta_{\gC}^+(\Sigma) = \delta_{\gC}^-(\Sigma)\} \\
        &= \{\Sigma \in C^0(G,\gC) \:|\: \rmapf{v}{e} \Sigma_{v} \rmapf{v}{e}^{\top} = \rmapf{u}{e} \Sigma_{u} \rmapf{u}{e}^{\top}\} \\
        &= \Gamma(G,\gC)
    \end{split}
\end{equation*}

A similar argument shows that $\varphi_{\delta_{\gM}}(\Sigma)_{ii} = \rmapf{v_i}{e_i} \Sigma_v \rmapf{v_i}{e_i}^{\top} + \rmapf{u_i}{e_i} \Sigma_u \rmapf{u_i}{e_i}^{\top}$.
\end{proof}

\subsection{Proof of \textbf{\texorpdfstring{\cref{thm:sdefLap}}{Theorem}}}
\textbf{Setup.} Before turning to the proof, we make explicit the decomposition $L_{\gC} = L_{\gC}^+ + L_{\gC}^-$ used in the statement. Using the decomposition $\delta_{\gM} = \delta_{\gM}^+ - \delta_{\gM}^-$ we get
\begin{equation*}
    L_{\gM} \!=\! \delta_{\gM}^{\top}\delta_{\gM} \!=\! \underbrace{\delta^{+^{\top}}_{\gM}\delta^+_{\gM} + \delta^{-^{\top}}_{\gM}\delta^-_{\gM}}_{L_{\gM}^+} - (\underbrace{\delta^{+^{\top}}_{\gM}\delta^-_{\gM} + \delta^{-^{\top}}_{\gM}\delta^+_{\gM}}_{L_{\gM}^-}).
\end{equation*}
Expanding the terms in the definition of $L_{\gC}$ and using the decomposition above, we can write $L_{\gC} = L_{\gC}^+ + L_{\gC}^-$ with
\begin{equation*}
\begin{split}
    &L_{\gC}^+(\Sigma) = \sum_{i,j} \varphi_{\mathbf{B}_i^{\phantom{+}}}\! \!\!\left( \varphi_{(\delta_{\gM}^+)^{\top}\mathbf{A}_j\delta_{\gM}^+}(\Sigma) + \varphi_{(\delta_{\gM}^-)^{\top}\mathbf{A}_j\delta_{\gM}^-}(\Sigma)\right), \\
    &L_{\gC}^-(\Sigma) = \sum_{i,j} \varphi_{\mathbf{B}_i^{\phantom{+}}}\!\!\!\left( \varphi_{(\delta_{\gM}^+)^{\top}\mathbf{A}_j\delta_{\gM}^-}(\Sigma) + 
    \varphi_{(\delta_{\gM}^-)^{\top}\mathbf{A}_j\delta_{\gM}^+}(\Sigma)\right).
\end{split}
\end{equation*}

\textbf{\Cref{thm:sdefLap}.}  Let $G = (V,E)$ be a connected graph and $(G, \gC)$ a \textbf{SDef}-valued sheaf. Then $\operatorname{Eq}(L_{\gC}^+, L_{\gC}^-) = H^0(G,\gC)$.

\begin{proof}
The equalizer is the subset of $C^0(G, \gC)$ for which
\begin{align*}
    \sum_{i=1}^n \varphi_{\mathbf{B}_i}\left(\sum_j \varphi_{(\delta_{\gM}^+)^{\top}\mathbf{A}_j\delta_{\gM}^+}(\Sigma) + \varphi_{(\delta_{\gM}^-)^{\top}\mathbf{A}_j\delta_{\gM}^-}(\Sigma)\right) = \sum_{i=1}^n \varphi_{\mathbf{B}_i}\left(\sum_j \varphi_{(\delta_{\gM}^+)^{\top}\mathbf{A}_j\delta_{\gM}^-}(\Sigma) + 
    \varphi_{(\delta_{\gM}^-)^{\top}\mathbf{A}_j\delta_{\gM}^+}(\Sigma)\right).
\end{align*}

Since each $\mathbf{B}_i$ has the effect of extracting the $i$-th (d-dimensional block) diagonal, we get for each $i \in [n]$ in the left hand side (lhs)
\begin{equation*}
    \left[\varphi_{(\delta_{\gM}^+)^{\top}}\left(\sum_j \varphi_{\mathbf{A}_j\delta_{\gM}^+}(\Sigma)\right) + \varphi_{(\delta_{\gM}^-)^{\top}}\left(\sum_j\varphi_{\mathbf{A}_j\delta_{\gM}^-}(\Sigma)\right)\right]_{ii},
\end{equation*}

and in the right hand side (rhs)
\begin{equation*}
    \left[\varphi_{(\delta_{\gM}^+)^{\top}}\left(\sum_j \varphi_{\mathbf{A}_j\delta_{\gM}^-}(\Sigma)\right) + \varphi_{(\delta_{\gM}^-)^{\top}}\left(\sum_j\varphi_{\mathbf{A}_j\delta_{\gM}^+}(\Sigma)\right)\right]_{ii}.
\end{equation*}

But what appears inside the parenthesis are exactly the definition of $\delta_{\gC}^+$ and $\delta_{\gC}^-$, thus
\begin{equation*}
    \left[\varphi_{(\delta_{\gM}^+)^{\top}}\left(\delta_{\gC}^+(\Sigma)\right) + \varphi_{(\delta_{\gM}^-)^{\top}}\left(\delta_{\gC}^-(\Sigma)\right)\right]_{ii} = \left[\varphi_{(\delta_{\gM}^+)^{\top}}\left(\delta_{\gC}^-(\Sigma)\right) + \varphi_{(\delta_{\gM}^-)^{\top}}\left(\delta_{\gC}^+(\Sigma)\right)\right]_{ii}.
\end{equation*}

Rearranging:
\begin{align}\label{eq1:eq+laplacian}
    \left[\varphi_{(\delta_{\gM}^+)^{\top}}\left(\delta_{\gC}^+(\Sigma)\right) - \varphi_{(\delta_{\gM}^+)^{\top}}\left(\delta_{\gC}^-(\Sigma)\right)\right]_{ii} &= \left[\varphi_{(\delta_{\gM}^-)^{\top}}\left(\delta_{\gC}^+(\Sigma)\right) - \varphi_{(\delta_{\gM}^-)^{\top}}\left(\delta_{\gC}^-(\Sigma)\right)\right]_{ii} \nonumber \\
    \left[\varphi_{(\delta_{\gM}^+)^{\top}}\left(\delta_{\gC}^+(\Sigma) -\delta_{\gC}^-(\Sigma)\right)\right]_{ii} &= \left[\varphi_{(\delta_{\gM}^-)^{\top}}\left(\delta_{\gC}^+(\Sigma) -\delta_{\gC}^-(\Sigma)\right)\right]_{ii}.
\end{align}

Since $\delta_{\gC}^+(\Sigma) -\delta_{\gC}^-(\Sigma)$ is a block-diagonal matrix, $(\delta_{\gM}^+)^{\top}(\delta_{\gC}^+(\Sigma) -\delta_{\gC}^-(\Sigma))$ will have the same structure as $(\delta_{\gM}^+)^{\top}$, i.e., the non-zero entries are in the exact same indices. Note that since we are only interested in the diagonal entries, Equation \ref{eq1:eq+laplacian} will be similar to an inner product of the i-th line of $(\delta_{\gM}^+)^{\top}(\delta_{\gC}^+(\Sigma) -\delta_{\gC}^-(\Sigma))$ and the i-th column of $\delta_{\gM}^+$. This is not exactly an inner product because our entries are matrices and the output of this operation is again a matrix. The same holds for the rhs.

The coboundary operator of the \textbf{Vect}-valued sheaf is defined using an arbitrary choice of direction for each edge. If $e \coloneqq (v,u) \in E$ and we choose $e = v \to u$, we get $\delta_{\gM}(\mu)_e = \rmapf{v}{e}\mu_v - \rmapf{u}{e}\mu_u$. Using this notion, we define $\delta_{\gM}^+(\mu)_e = \rmapf{v}{e}\mu_v$ and $\delta_{\gM}^-(\mu)_e = \rmapf{u}{e}\mu_u$. This extends to all vertices and edges on the graph. Now, similar to the proof of \cref{thm:sdefH0} we can label the lines of $\delta_{\gM}$ by edges and the columns by vertices, and then the coboundary has only two entries per line. This implies that the positive and negative parts will have only one entry at each line. Moreover, using the notation described in the proof of \cref{thm:sdefH0}, if an edge $e$ has node $v$ as source, then $(\delta_{\gM}^+)_e^v = \rmapf{v}{e}$, $(\delta_{\gM}^-)_e^v = 0_{d \times d}$ and $(\delta_{\gM}^+)_e^v = 0_{d \times d}$, $(\delta_{\gM}^-)_e^v = \rmapf{u}{e}$.

Since the orientation of the edges is arbitrary, our results must be independent of these choices. We have $n = |V|$ equations in the form of  \Cref{eq1:eq+laplacian}. Each index can be identified with a node and from now on we fix a node $v \in V$, since all equations have the same structure and depend only on the column $v$ of the positive and negative parts of the coboundary map (resp. the line $v$ in their transposes), as we mentioned before. We have $2^{|E|}$ possible orientations for the edges, and we will denote $E_q$ the set of oriented edge for some $q \in \left[2^{|E|}\right]$. Notice that the relevant orientations in this case are only those related to $v$, thus we can restrict ourselves to a subset $E^v = \{E_{q_1}^v\, \dots, E_{q_{2^{|N(v)|}}}^v\}$ of orientations that only accounts for $v$ and its neighbors. Let $S_v$ be the solution set of  \cref{eq1:eq+laplacian} for our fixed node $v$. Then $S_v = \bigcap_{r} S_v^{E_{q_r}}$, where $S_v^{E_{q_r}}$ is the specific solution set for the given oriented edge set $E_{q_r}$. We will omit the superscript $v$ in the oriented edge set since $S$ already has it as subscript. 
 
Consider that for our fixed node $v$, we choose all its incident edges to have $v$ as source, and assume that this orientation is given by the first element in $E^v$. Then $\forall e \in E_{q_1}^v$ s.t. $v \unlhd e$ we have $(\delta_{\gM}^+)_e^v = \rmapf{v}{e}$, and consequently 
$(\delta_{\gM}^-)_e^v = 0_{d \times d}$. Therefore the column $v$ in $\delta_{\gM}^-$ is composed only by zeros, the remaining edges that are not incident to $v$ clearly have null entry at this column. For simplicity, assume that the column indexed by $v$ is the first one. 

We pointed out that both sides of \cref{eq1:eq+laplacian} have a description that resembles an inner product between the lines and columns of the corresponding matrices, in particular, for $(\delta_{\gM}^-)^{\top}(\delta_{\gC}^+(\Sigma) -\delta_{\gC}^-(\Sigma))$ and $\delta_{\gM}^-$. The first column, indexed by $v$, in $\delta_{\gM}^-$ is zero. Using only subscripts since the result is a square matrix, that can be indexed by vertices, we have:
\begin{equation*}
    \left[\varphi_{(\delta_{\gM}^-)^{\top}}\left(\delta_{\gC}^+(\Sigma) -\delta_{\gC}^-(\Sigma)\right)\right]_{vv} = [(\delta_{\gM}^-)^{\top}(\delta_{\gC}^+(\Sigma) - \delta_{\gC}^-(\Sigma))]_v [\delta_{\gM}^-]_v = 0.
\end{equation*}

Hence, \cref{eq1:eq+laplacian} can be rewritten as
\begin{equation*}
    \left[\varphi_{(\delta_{\gM}^+)^{\top}}\left(\delta_{\gC}^+(\Sigma) -\delta_{\gC}^-(\Sigma)\right)\right]_{vv} = 0 \implies [(\delta_{\gM}^+)^{\top}(\delta_{\gC}^+(\Sigma) - \delta_{\gC}^-(\Sigma))]_v [\delta_{\gM}^+]_v = 0.
\end{equation*}

Let $\{e_j\}_{j \in [|E|]}$ be some ordering of the edges such that the first $|N(v)|$ edges are incident to $v$ (the lines/columns of $\delta_{\gC}^+(\Sigma) -\delta_{\gC}^-(\Sigma)$ are $e_1, \dots, e_{|N(v)|}, \dots, e_{|E|}$). Also let

\begin{align*}
    \delta_{\gC}^+(\Sigma) = \operatorname{block-diag}(\delta^+_1, \dots, \delta^+_{|E|}) \\
    \delta_{\gC}^-(\Sigma) = \operatorname{block-diag}(\delta^-_1, \dots, \delta^-_{|E|}).
\end{align*}

Then
\begin{equation}
    [(\delta_{\gM}^+)^{\top}(\delta_{\gC}^+(\Sigma) - \delta_{\gC}^-(\Sigma))]_v [\delta_{\gM}^+]_v = \sum_{j=1}^{|N(v)|} \rmapf{v}{e_j}^{\top}(\delta^+_j-\delta^-_j)\rmapf{v}{e_j} = 0.
\end{equation}

Therefore, the solution set $S_v^{E_{q_1}}$ of \cref{eq1:eq+laplacian} can be described as
\begin{equation}\label{eq4}
    \sum_{j=2}^{|N(v)|} \rmapf{v}{e_j}^{\top}(\delta^+_j-\delta^-_j)\rmapf{v}{e_j} = -\rmapf{v}{e_1}^{\top}(\delta^+_1-\delta^-_1)\rmapf{v}{e_1}.
\end{equation}

Now, we consider another orientation $E_{q_2}$ for the edges incident to $v$: using our numbering of the edges above, we only change the orientation of the first edge $e_1$. Thus, the only nonzero entry in the column $v$ of $\delta^-_{\gM}$ is $\rmapf{v}{e_1}$ and we get on the rhs of equation \ref{eq1:eq+laplacian}:
\begin{align*}
    \left[\varphi_{(\delta_{\gM}^-)^{\top}}\left(\delta_{\gC}^+(\Sigma) -\delta_{\gC}^-(\Sigma)\right)\right]_{vv} = [(\delta_{\gM}^-)^{\top}(\delta_{\gC}^+(\Sigma) - \delta_{\gC}^-(\Sigma))]_v [\delta_{\gM}^-]_v = \rmapf{v}{e_1}^{\top}(\delta^+_1 - \delta^-_1)\rmapf{v}{e_1}.
\end{align*}

The lhs is the lhs of \cref{eq4}. Hence:
\begin{equation}\label{eq5}
    \sum_{j=2}^{|N(v)|} \rmapf{v}{e_j}^{\top}(\delta^+_j-\delta^-_j)\rmapf{v}{e_j} = \rmapf{v}{e_1}^{\top}(\delta^+_1-\delta^-_1)\rmapf{v}{e_1}.
\end{equation}

From Equations \ref{eq4} and \ref{eq5}, i.e. the intersection $S_v^{E_{q_1}} \cap S_v^{E_{q_2}}$, we conclude that
\begin{equation*}
    \rmapf{v}{e_1}^{\top}(\delta^+_1-\delta^-_1)\rmapf{v}{e_1} = 0.
\end{equation*}

Since all restriction maps in the sheaf are invertible this is equivalent to $\delta^+_1 = \delta^-_1$. Notice that  \cref{eq4} can be actually rewritten as
\begin{equation*}
    \sum_{\substack{j=1 \\ j\neq p}}^{|N(v)|} \rmapf{v}{e_j}^{\top}(\delta^+_j-\delta^-_j)\rmapf{v}{e_j} = -\rmapf{v}{e_p}^{\top}(\delta^+_p-\delta^-_p)\rmapf{v}{e_p},\: \forall p \in [|N(v)|],
\end{equation*}

and we can choose an orientation $E_{q_{p+1}}$ as we did for $E_{q_2}$ to conclude that the intersection $S_v^{E_{q_1}} \cap S_v^{E_{q_{p+1}}}$ is
\begin{equation*}
    \rmapf{v}{e_p}^{\top}(\delta^+_p-\delta^-_p)\rmapf{v}{e_p} = 0 \implies \delta^+_p = \delta^-_p, \forall p \in [|N(v)|].
\end{equation*}

Then we have
\begin{equation*}
    \bigcap_{r=1}^{|N(v)|+1} S_v^{E_{q_r}} = \{\delta^+_j = \delta^-_j, \forall j \in [|N(v)|\}.
\end{equation*}

Consequently, 
\begin{equation*}
    \{\delta^+_j = \delta^-_j, \forall j \in [|N(v)|\} \subset \bigcap_{r = |N(v)|+2}^{2^{|N(v)|}} S_v^{E_{q_r}},
\end{equation*}

and we can conclude that $S_v = \{\delta^+_j = \delta^-_j, \forall j \in [|N(v)|\}$.

Our fixed node $v$ was arbitrarily chosen, therefore we can extend the same argument to other nodes until all edges are used, obtaining the final solution 
\begin{equation*}
    \operatorname{Eq}(L_{\gC}^+, L_{\gC}^-) = \bigcup_{v \in V} S_v = \{\delta^+_j = \delta^-_j, \forall j \in [|E|]\} = \operatorname{Eq}(\delta_{\gC}^+, \delta_{\gC}^-).
\end{equation*}
\end{proof}

\subsection{Proof of \textbf{\texorpdfstring{\Cref{prop:orbitorth}}{Proposition}}}

\textbf{\Cref{prop:orbitorth}.} A Gaussian $\varrho \in \mathcal{G}(\mathbb{R}^d)$ with parameters $(\mu', \Sigma')$ is in the orbit of $\vartheta$ under the action of $O(d)$ only if $\|\mu'\|_2 = \|\mu\|_2$ and $\Sigma'$ has the same eigenvalues as $\Sigma$.

\begin{proof}
Let $\nu = \gN(\mu, \Sigma)$, thus $\varrho \in O(d) \cdot \nu$ is of the form $\gN(Q\mu, Q\Sigma Q^{\top})$ with $Q \in O(d)$. Therefore $\|Q\mu\|_2 = \|\mu\|_2$ and $\Sigma' = Q\Sigma Q^{\top} = Q\Sigma Q^{-1}$ is similar to $\Sigma$.
\end{proof}

\subsection{Proof of \textbf{\texorpdfstring{\Cref{prop:staborth}}{Proposition}} }
\textbf{\Cref{prop:staborth}.}    Let $\lambda_1, \dots, \lambda_k$ be the eigenvalues of $\Sigma$ with multiplicities $m_1, \dots, m_k$. If we write $\mu = (\mu_1, \dots, \mu_k)$ where $\mu_j \in \sR^{m_j}$, the stabilizer of $\vartheta$ is in bijection with $\prod_{j=1}^k O(m_j-\mathbbm{1}_{\{\mu_j \neq 0\}})$. If all eigenvalues are different, the stabilizer is in bijection with $(\sZ_2)^{|\{j \:|\: \mu_j = 0\}|}$.

\begin{proof}
Let $\nu = \gN(\mu, \Sigma)$. We will study the stabilizer of the parameters separately.
For the covariances, we first consider the simplest case where $\Sigma = \operatorname{diag}(\lambda_1, \dots, \lambda_n)$.

If all eigenvalues are different, we get $\operatorname{Stab}(\Sigma) = \{\operatorname{diag}(\alpha_1, \dots, \alpha_d), \alpha_i \in \{\pm 1\}\} \simeq (\sZ_2)^d$. 

If we have $k$ eigenvalues with multiplicities $m_1, \dots, m_k$, then $\Sigma = \bigoplus_{j=1}^k \lambda_jI_{m_j}$. On one hand, if $Q\Sigma Q^{\top} = \Sigma$, then $Q$ must preserve each block, i.e. $Q = \bigoplus_{j=1}^k Q_j$, with $Q_j \in O(m_j)$. On the other hand, if $Q = \bigoplus_{j=1}^k Q_j$, with $Q_j \in O(m_j)$:
\begin{equation*}
    Q\Sigma Q^{\top} = \bigoplus_{j=1}^k Q_j \lambda_j I_{m_j} Q_j^{\top} = \bigoplus_{j=1}^k \lambda_j Q_jQ_j^{\top} = \bigoplus_{j=1}^k \lambda_j I_{m_j} = \Sigma.
\end{equation*}

Therefore, $\operatorname{Stab}(\Sigma) \cong \prod_{j=1}^k O(m_j) $. For the general setting, we know that $\Sigma$ can be written as $V \Lambda V^{\top}$ with $V \in O(d)$, thus \looseness=-1
\begin{equation*}
    Q\Sigma Q^{\top} = QV\Lambda V^{\top} Q^{\top} = V\Lambda V^{\top} \implies (V^{\top} Q V) \Lambda (V^{\top} Q V)^{\top} = \Lambda,
\end{equation*}

and taking $U = V^{\top}QV \in O(d)$ we have $\operatorname{Stab}(\Sigma) \cong \operatorname{Stab}(\Lambda)$.

Now, for the means, we recall that the stabilizer of a vector $\mu \in \sR^d$ is isomorphic to $O(d-1)$ if $\mu \neq 0$ and to $O(d)$ if $\mu = 0$. If we decompose $\mu = (\mu_{1},...,\mu_{k})$ where $\mu_{j} \in \sR^{m_j}$, since $Q \in \operatorname{Stab}(\Sigma)$ is of the form $Q = \bigoplus_{j=1}^k Q_j$ with $Q_j \in O(m_j)$, we get that
\begin{equation*}
    Q\mu = \mu \iff Q_j\mu_j = \mu_j,\: \forall j \in \{1,\dots,k\},
\end{equation*}

thus $\operatorname{Stab}(\nu) = \operatorname{Stab}(\Sigma) \cap \operatorname{Stab}(\mu) = \prod_{j=1}^k O\left(m_j - \mathbbm{1}_{\{\mu_j \neq 0\}}\right)$.

If all eigenvalues are different, $\operatorname{Stab}(\nu) = (\sZ_2)^{|\{j \:|\: \mu_j = 0\}|}$
\end{proof}

\subsection{Proof of \textbf{\texorpdfstring{\Cref{prop:section_existence}}{Proposition}}}

\begin{proposition}\label{prop:cycleorth}   
Let $(G,\gF)$ be a Gaussian sheaf with orthogonal restriction maps over a connected graph $G$. Then for any $\nu \in H^0(G,\gF)$ and cycle $\gamma$ based at $v \in V$ we have $\mathbf{P}^{\gamma}_{v \to v}(\nu_v) = \nu_v$.
\end{proposition}

\begin{proof}
Let $\nu \in H^0(G,\gF)$ and $\gamma = (v_0, v_1, \dots, v_{\ell}, v_{\ell+1})$ with $v_0 = v_{\ell+1} = v$. Then letting $e_i \coloneqq (v_i, v_{i+1})$ we have
\begin{equation*}
        \rmapf{v_{i+1}}{e_i}\mu_{v_{i+1}} = \rmapf{v_i}{e_i}\mu_{v_i} \implies \mu_{v_{i+1}} = \rmapf{v_{i+1}}{e_i}^{\top}\rmapf{v_i}{e_i}\mu_{v_{i}}
\end{equation*}
for the means and
\begin{equation*}
\begin{split}
    \rmapc{v_{i+1}}{e_i}(\Sigma_{v_{i+1}}) = \rmapc{v_i}{e_i}(\Sigma_{v_i})  &\implies \rmapf{v_{i+1}}{e_i}\Sigma_{v_{i+1}}\rmapf{v_{i+1}}{e_i}^{\top} = \rmapf{v_i}{e_i}\Sigma_{v_i}\rmapf{v_i}{e_i}^{\top} \\ &\implies \Sigma_{v_{i+1}} = \rmapf{v_{i+1}}{e_i}^{\top}\rmapf{v_i}{e_i}\Sigma_{v_i}\rmapf{v_i}{e_i}^{\top}\rmapf{v_{i+1}}{e_i} \\&\implies \Sigma_{v_{i+1}} = \varphi_{\rmapf{v_{i+1}}{e_i}^{\top}\rmapf{v_i}{e_i}}(\Sigma_{v_{i}})
\end{split}
\end{equation*}
for the covariances. Hence the transport
\begin{equation*}
    \begin{split}
        \mathbf{P}^{\gamma}_{v \to v}(\nu_v) &=  \left(\Phi_{\rmapf{v}{e_{\ell}}^{\top}} \circ \Phi_{\rmapf{v_{\ell}}{e_{\ell}}}\right) \circ \ldots \circ \left(\Phi_{\rmapf{v_1}{e_{0}}^{\top}} \circ \Phi_{\rmapf{v}{e_{0}}}\right)(\nu_v)  \\ &= \Phi_{\rmapf{v}{e_{\ell}}^{\top}\rmapf{v_{\ell}}{e_{\ell}}} \circ \ldots \circ \left(\Phi_{\rmapf{v_1}{e_{0}}^{\top}\rmapf{v}{e_{0}}}(\nu_v)\right)
    \end{split}
\end{equation*}
can be described by the following operations over the parameters (consider $\mathbf{P}^{\gamma}_{v \to v}(\nu_v) = \gN(\mu^{\ast},\Sigma^{\ast})$) 
\begin{equation*}
\begin{split}
    \mu^{\ast} &= \rmapf{v}{e_{\ell}}^{\top}\rmapf{v_{\ell}}{e_{\ell}} \cdots \rmapf{v_1}{e_{0}}^{\top}\rmapf{v}{e_{0}}\mu_v \\
        \Sigma^{\ast} &= \varphi_{\rmapf{v}{e_{\ell}}^{\top}\rmapf{v_{\ell}}{e_{\ell}}}\circ \cdots \circ \left(\varphi_{\rmapf{v_1}{e_{0}}^{\top}\rmapf{v}{e_{0}}^{~}}(\Sigma_v)\right)
\end{split}
\end{equation*}

Using the global section property obtained before we conclude that $\mu^{\ast} = \mu_v$ and $\Sigma^{\ast} = \Sigma_v$, i.e. $\mathbf{P}^{\gamma}_{v \to v}(\nu_v) = \nu_v$.
\end{proof}

\textbf{\texorpdfstring{\Cref{prop:section_existence}}{Proposition}} Let $(G,\gF)$ be a Gaussian sheaf with $O(d)$ maps over a connected graph $G$. Then $\exists \nu \in H^0(G,\gF)$ with $\nu_v = \vartheta$ if, and only if, for every cycle $\gamma$ based at $v$ holds that $\mathbf{P}^{\gamma}_{v \to v} \in \operatorname{Stab}(\vartheta)$.

\begin{proof}
    Let $\nu \in H^0(G,\gF)$ with $\nu_v = \vartheta$. Then for every loop $\gamma$ based at $v$, we get $\mathbf{P}^{\gamma}_{v \to v}(\vartheta) = \vartheta$, i.e. $\mathbf{P}^{\gamma}_{v \to v} \in \operatorname{Stab}(\vartheta)$.

    Consider a path $\varsigma: v \to w$, and define $\nu_w = \mathbf{P}^{\varsigma}_{v \to w}(\vartheta)$. Take another path $\varpi: v \to w$, and define the loop $\gamma = \varpi^{-1} \circ \varsigma$, thus $\mathbf{P}^{\varsigma}_{v \to w} = \mathbf{P}^{\varpi}_{v \to w} \circ \mathbf{P}^{\gamma}_{v \to v}$. If $\mathbf{P}^{\gamma}_{v \to v} \in \operatorname{Stab}(\vartheta)$ for every loop $\gamma$ at $v$, then
    \begin{equation*}
        \nu_w = \mathbf{P}^{\varsigma}_{v \to w}(\vartheta) = \mathbf{P}^{\varpi}_{v \to w}(\mathbf{P}^{\gamma}_{v \to v}(\vartheta)) = \mathbf{P}^{\varpi}_{v \to w}(\vartheta),
    \end{equation*}
    
    thus we have a well-defined map $\nu : V(G) \to \gG(\sR^d)$. To see that it is a global section, consider an edge $e = (w,u) \in E(G)$, and the unit-length path $\varkappa: w \to u$. Thus $\nu_w = \mathbf{P}^{\varsigma}_{v \to w}(\vartheta)$ and $\nu_u =\mathbf{P}^{\varkappa \circ \varsigma}_{v \to u}(\vartheta) = \mathbf{P}^{\varkappa}_{w \to u}(\mathbf{P}^{\varsigma}_{v \to w}(\vartheta)) = \mathbf{P}^{\varkappa}_{w \to u}(\nu_w)$. Recall that $\mathbf{P}^{\varkappa}_{w \to u} = \Phi_{\rmapf{u}{e}^{\top}\rmapf{w}{e}}$, therefore
    \begin{equation*}
        \mathbf{P}^{\varkappa}_{w \to u}(\nu_w) = \gN\left(\rmapf{u}{e}^{\top}\rmapf{w}{e}\mu_w, (\rmapf{u}{e}^{\top}\rmapf{w}{e})\Sigma_w(\rmapf{u}{e}^{\top}\rmapf{w}{e})^{\top}\right) = \gN(\mu_u, \Sigma_u).
    \end{equation*}
    
    This yields a system of equations, and since the restriction maps are orthogonal
    \begin{equation*}
        \begin{cases}
            \mu_u = \rmapf{u}{e}^{\top}\rmapf{w}{e}\mu_w \\
            \Sigma_u = (\rmapf{u}{e}^{\top}\rmapf{w}{e})\Sigma_w(\rmapf{u}{e}^{\top}\rmapf{w}{e})^{\top}
        \end{cases} \implies
        \begin{cases}
            \rmapf{u}{e}\mu_u = \rmapf{w}{e}\mu_w \\
            \rmapf{u}{e}\Sigma_u\rmapf{u}{e}^{\top} = \rmapf{w}{e}\Sigma_w\rmapf{w}{e}^{\top}
        \end{cases}.
    \end{equation*}
    
    We conclude that $\nu = (\nu_1, \dots, \nu_n) \in H^0(G,\gF)$.
\end{proof}

\subsection{Proof of \textbf{\texorpdfstring{\Cref{prop:section_bijection}}{Proposition}}}
\textbf{\Cref{prop:section_bijection}.} Let $(G,\gF)$ be a Gaussian sheaf with $O(d)$ maps over a connected graph $G$ where the transport is path-independent. Then $H^0(G,\gF)$ and $\gG(\sR^d)$ are in bijection.

\begin{proof}
    For each $\nu \in H^0(G,\gF)$, we can define maps $f_i : H^0(G,\gF) \to \gG(\sR^d)$ with $f_i(\nu) = \nu_i = \gN(\mu_i, \Sigma_i)$, the Gaussian associated with each node $v_i \in V(G)$. Fix $i = k$ and consider the function $f_k$.

    Injectivity of $f_k$: Let $\nu, \nu' \in H^0(G,\gF)$ with $f_k(\nu) = f_k(\nu')$. For any $v_j \in V$, choose $\gamma: v_k \to v_j$ path, thus
    \begin{equation*}
        \nu_j = f_j(\nu) = \mathbf{P}^{\gamma}_{v_k \to v_j}(f_k(\nu)) = \mathbf{P}^{\gamma}_{v_k \to v_j}(f_k(\nu')) = \nu_j'.
    \end{equation*}

    Since this holds $\forall j \in [n] \setminus \{k\}$ and $f_k(\nu) = f_k(\nu')$, we get $\nu = \nu'$, i.e., $f_k$ is injective.

    Surjectivity of $f_k$: Let $\vartheta \in \gG(\sR^d)$ and define $g : V(G) \to \gG(\sR)^d$ by $g(v_j) = \mathbf{P}^{\gamma}_{v_k \to v_j}(\vartheta)$ for a path $\gamma: v_k \to v_j$. Notice that this is well-defined due to path-independence. Let $e = (v_{\ell}, v_m)$ and $\varsigma: v_k \to v_{\ell}$, then $\varpi: \varkappa \circ \varsigma$, where $\varkappa: v_{\ell} \to v_m$ is the unit-length path from $v_{\ell}$ to $v_m$. Therefore, since $g(v_{\ell}) = \mathbf{P}^{\varsigma}_{v_k \to v_{\ell}}(\vartheta)$:
    \begin{equation*}
        g(v_m) = \mathbf{P}^{\varpi}_{v_k \to v_m}(\vartheta) = \mathbf{P}^{\varkappa}_{v_{\ell} \to v_{m}}(\mathbf{P}^{\varsigma}_{v_k \to v_{\ell}}(\vartheta)) = \mathbf{P}^{\varkappa}_{v_{\ell} \to v_{m}}(g(v_{\ell})).
    \end{equation*}

    Consequently,  $\nu = (g(v_1), \dots, g(v_{k-1}), \vartheta, g(v_{k+1}), \dots, g(v_n)) \in H^0(G,\gF)$. Thus, $\forall \vartheta \in \gG(\sR^d)$, $\exists \nu \in H^0(G,\gF)$ as constructed above with $f_k(\nu) = \vartheta$. So $f_k$ is surjective, as desired.

    We conclude that $f_k: H^0(G,\gF) \to \gG(\sR^d)$ is a bijection.
\end{proof}

\subsection{Proof of \textbf{\texorpdfstring{\Cref{prop:orbitdiag}}{Proposition}}}
\textbf{\Cref{prop:orbitdiag}.}     The orbit of $\vartheta$ under the action of $D(d)$ is isomorphic to $D(d)$ and the stabilizer is always trivial, if $\mu_i \neq 0$ for all $i = 1,\dots, n$.   

\begin{proof}
Let $\nu = \gN(\mu, \Sigma)$. Each mean vector $\mu' = (\mu'_1, \dots, \mu'_n)$ uniquely determines a matrix $D = \operatorname{diag}\left(\frac{\mu'_1}{\mu_1}, \dots, \frac{\mu'_n}{\mu_n}\right)$ such that $D\mu = \mu'$. Thus given a mean vector, the covariance matrix is completely determined: $(D\Sigma D^{\top})_{ij} = \frac{\mu'_i \mu'_j}{\mu_i \mu_j} \Sigma_{ij}$. Also for a general matrix $D \in D(d)$, we have that $D\mu = (d_1\mu_1, \dots, d_n\mu_n)$ is in the stabilizer only if $d_i = 1$, $\forall i \in [d]$
\end{proof}

\subsection{Proof of \textbf{\texorpdfstring{\Cref{prop:orbitgen}}{Proposition}}}

\textbf{\Cref{prop:orbitgen}.}  Under the action of $GL(d)$ we can achieve any mean vector by restricting the covariances to a specific set. Moreover, if $\Sigma$ is positive-definite, by restricting the means to a specific set then we can also achieve any positive-definite covariance matrix. 
\begin{proof}
Let $\nu = \gN(\mu, \Sigma)$. Given any mean vector $\mu' \in \sR^d$, there exist a family of matrices $A \in GL(d)$ satisfying $A\mu = \mu'$. Thus we restrict the covariances to the subset of PSD matrices that can be achieved through this family. Also if $\Sigma$ is positive definite, there exists decomposition $\Sigma = B^{\top} B$ with $B \in GL(d)$ (Cholesky, square root...), and for any covariance matrix $\Sigma'$ that is also positive definite with $A\Sigma A^{\top} = \Sigma'$ and $\Sigma' = C^{\top}C$, then $A = C^{\top} Q (B^{\top})^{-1}$ with $Q \in O(d)$. Now we limit the set of possible mean vectors to the images of $\mu$ under the action of this family parameterized by $C$, $Q$ and $B$.
\end{proof}

\subsection{\textbf{\texorpdfstring{\Cref{prop:lyapenergy}}{Proposition}.}}
\textbf{\Cref{prop:lyapenergy}}     Let $G$ be a connected graph and $(G,\gF)$ a Gaussian sheaf with orthogonal restriction maps. Then there exists a family $\{\nu_j\}_{j \in J} \subset C^0(G,\gF)$ and weight matrices $W_{\alpha}$, $W_{\beta}$ such that $V((I \otimes W_{\alpha})\#\nu) > V(\nu)$ and $V((I \otimes W_{\beta})\#\nu) < V(\nu)$.

\begin{proof}
Let $\nu = \gN(\mu, \Sigma)$, with $\Sigma_v = \sigma_vI$, and $W = cI$. In cellular sheaves over a graph with orthogonal restriction maps we have $D_v = \operatorname{deg}(v)I$, $\forall v \in V$. We analyze a single parcel of the energy $V(\nu)$ in the edge $e\coloneqq(v,u)\in E$ for the covariances. Using the linearity of the trace we get
\begin{align*}
    \operatorname{tr}\left(\Sigma_v' + \Sigma_u' - 2\left(\Sigma_v'^{\frac{1}{2}}\Sigma_u'\Sigma_v'^{\frac{1}{2}}\right)^{\frac{1}{2}}\right) = \operatorname{tr}(\Sigma'_v) + \operatorname{tr}(\Sigma'_u) -2\operatorname{tr}\left(\left(\Sigma_v'^{\frac{1}{2}}\Sigma_u'\Sigma_v'^{\frac{1}{2}}\right)^{\frac{1}{2}}\right).
\end{align*}

The parcel $\operatorname{tr}(\Sigma'_v)$ is
\begin{align*}
    \operatorname{tr}(\Sigma'_v) &= \operatorname{tr}(\rmapf{v}{e}D_v^{-\frac{1}{2}}(\sigma_vI)D_v^{-\frac{1}{2}}\rmapf{v}{e}^{\top}) \\ 
    &= \sigma_v\operatorname{tr}(\rmapf{v}{e}D_v^{-1}\rmapf{v}{e}^{\top}) \\
    &= \sigma_v\operatorname{tr}(\rmapf{v}{e}(\operatorname{deg}(v)^{-1}I)\rmapf{v}{e}^{\top}) \\
    &= \sigma_v\operatorname{deg}(v)^{-1}\operatorname{tr}(\rmapf{v}{e}\rmapf{v}{e}^{\top}) \\
    &= \sigma_v\operatorname{deg}(v)^{-1}\operatorname{tr}(I) \\ 
    &= d\frac{\sigma_v}{\operatorname{deg}(v)},
\end{align*}

hence $\operatorname{tr}(\Sigma''_u) = d\frac{\sigma_u}{\operatorname{deg}(u)}$. For the last parcel:
\begin{align*}
    \operatorname{tr}\left(\left(\Sigma_v'^{\frac{1}{2}}\Sigma_u'\Sigma_v'^{\frac{1}{2}}\right)^{\frac{1}{2}}\right) &= \operatorname{tr}\left(\left(\Sigma'_u\Sigma'_v\right)^{\frac{1}{2}}\right) \\ 
    &= \operatorname{tr}\left(\left(\rmapf{v}{e}D_v^{-\frac{1}{2}}(\sigma_vI)D_v^{-\frac{1}{2}}\rmapf{v}{e}^{\top}\rmapf{u}{e}D_u^{-\frac{1}{2}}(\sigma_uI)D_u^{-\frac{1}{2}}\rmapf{u}{e}^{\top}\right)^{\frac{1}{2}}\right) \\ 
    &= (\sigma_v\sigma_u)^{\frac{1}{2}}\operatorname{tr}\left(\left(\rmapf{v}{e}(\operatorname{deg}(v)^{-1}I)\rmapf{v}{e}^{\top}\rmapf{u}{e}(\operatorname{deg}(u)^{-1}I)\rmapf{u}{e}^{\top}\right)^{\frac{1}{2}}\right) \\
    &= (\sigma_v\sigma_u)^{\frac{1}{2}}(\operatorname{deg}(v)\operatorname{deg}(u))^{-\frac{1}{2}}\operatorname{tr}\left((\rmapf{v}{e}\rmapf{v}{e}^{\top}\rmapf{u}{e}\rmapf{u}{e}^{\top})^{\frac{1}{2}}\right) \\
    &= (\sigma_v\sigma_u)^{\frac{1}{2}}(\operatorname{deg}(v)\operatorname{deg}(u))^{-\frac{1}{2}}\operatorname{tr}(I) \\
    &= d\sqrt{\frac{\sigma_v\sigma_u}{\operatorname{deg}(v)\operatorname{deg}(u)}}.
\end{align*}

We conclude that
\begin{equation*}
    \operatorname{tr}\left(\Sigma_v' + \Sigma_u' - 2\left(\Sigma_v'^{\frac{1}{2}}\Sigma_u'\Sigma_v'^{\frac{1}{2}}\right)^{\frac{1}{2}}\right) = d\left(\frac{\sigma_v}{\operatorname{deg}(v)} + \frac{\sigma_u}{\operatorname{deg}(u)} - 2\sqrt{\frac{\sigma_v\sigma_u}{\operatorname{deg}(v)\operatorname{deg}(u)}}\right)
\end{equation*}

For the energy in the next layer, $V((I \otimes W)\#\nu)$, letting $\Sigma_v'' = \varphi_{\rmapf{v}{e}D_v^{-\frac{1}{2}}W}(\Sigma_v)$ a similar calculation yields
\begin{align*}
    \operatorname{tr}\left(\Sigma_v'' + \Sigma_u'' - 2\left(\Sigma_v''^{\frac{1}{2}}\Sigma_u''\Sigma_v''^{\frac{1}{2}}\right)^{\frac{1}{2}}\right) &= c^2d\left(\frac{\sigma_v}{\operatorname{deg}(v)} + \frac{\sigma_u}{\operatorname{deg}(u)} - 2\sqrt{\frac{\sigma_v\sigma_u}{\operatorname{deg}(v)\operatorname{deg}(u)}}\right) \\
    &= c^2\operatorname{tr}\left(\Sigma_v' + \Sigma_u' - 2\left(\Sigma_v'^{\frac{1}{2}}\Sigma_u'\Sigma_v'^{\frac{1}{2}}\right)^{\frac{1}{2}}\right)
\end{align*}

Now for the means, from $E_{\gM}(\mu) = \mu^{\top}\Delta_{\gM}\mu$ we can easily conclude that $E_{\gM}((I \otimes W)\mu) = c^2E(\mu)$.

Putting everything together, $V((I \otimes W)\#\nu) = c^2V(\nu)$. Therefore, choosing $c \in (0,1)$ we get $V((I \otimes W)\#\nu) < V(\nu)$, and choosing $c > 1$ yields the opposite inequality.
\end{proof}

\section{A categorical definition of sheaves: why do we care about the target category}
\label{sec:app:catshv}
In this section, we give an introduction to the main definitions in category theory that lead us to the definition of sheaves whose domain are a category of open sets of some topological space. Then, we present cellular sheaves for a cell complex and explain how both notions are deeply connected.
\begin{definition}
    A category \textbf{C} consists of a collection of \textit{objects} and a collection of \textit{morphism} between pair of objects such that
    \begin{enumerate}
        \item If $f : A \to B$ and $g : C\to B$ are morphism of \textbf{C}, there must exist a unique morphism of \textbf{C}, called composition, $g \circ f : A \to C$. Moreover, the composition is associative, i.e, for any three composable morphisms $f, g, h$ it holds that $f\circ (g \circ h) = (f\circ g)\circ h$.
        \item For every object $A$ of  \textbf{C}, there must exist a unique morphism of \textbf{C}, called identity, $id_A : A \to A$. Moreover, it must satisfy that for any morphism $f: A \to B$ we have $f \circ id_A = f = id_B \circ f$.
    \end{enumerate}
\end{definition}
\begin{example}
    \textbf{Set} is the category whose objects are sets, and morphisms are functions between them. \textbf{Vect} is the category whose objects are vector spaces, and morphisms are linear maps between them. 
\end{example}

\begin{example}
    Any poset $(P,\leq)$ (set with partial order) can be seem as a category whose objects are elements of $P$ and for any two objects $A, B$, there is a unique morphism $A \to B$ if $A \leq B$. 
\end{example}

\begin{example}
    Given a category \textbf{C}, we can define its opposite category $\textbf{C}^{op}$ by reversing the direction of the morphisms. In other words, $A \to B$ is a morphism of \textbf{C} iff $B \to A$ is a morphism of $\textbf{C}^{op}$.
\end{example}

\begin{proposition}
    \textbf{SDef} is a category.
\end{proposition}
\begin{proof}
    Recall \textbf{SDef} denotes the category of positive semidefinite matrices whose objects are the (convex) cones of positive semidefinite matrices. A map $\varphi : C \to C'$ in \textbf{SDef} is defined by a unique linear operator $A$ through $\varphi_A(\Sigma) = A\Sigma A^{\top}$.

    Since $\varphi_A \circ \varphi_B 
 = \varphi_{AB}$, we have:  $$(\varphi_A \circ \varphi_B) \circ \varphi_C = \varphi_{AB} \circ \varphi_C = \varphi_{ABC} = \varphi_{A}\circ \varphi_{BC}    = \varphi_A \circ (\varphi_B \circ \varphi_C)$$ 
 So the composition is associative. If we take sum of maps in \textbf{SDef}, the composition still is associative. 

 For any cone $C$ of positive semidefinite matrices, the identity map is uniquely determined by the identity matrix, since $\varphi_I(\Sigma) = I\Sigma I^{\top} = \Sigma$, for every $\Sigma \in C$. 
\end{proof}

\begin{definition}
    Given \textbf{C} and \textbf{D} two categories. A functor $F: \textbf{C} \to \textbf{D}$  maps
each object and morphism of \textbf{C} to a correspondent object and morphism of \textbf{D} in such way that:
\begin{enumerate}
    \item $F(id_A) = id_{F(A)}$ for every object $A$ of \textbf{C}.
    \item $F(g \circ f) = F(g) \circ F(f)$ for every composable morphisms $f,g$ of \textbf{C}.
\end{enumerate}
\end{definition}
\begin{example}
    Let $X$ be a topological space and denote by $\mathcal{O}(X)$ the category associated with the poset  of all open sets of $X$, i.e. whose objects are open sets of $X$ and whose morphisms are given by the inclusion relation. Consider $F:\mathcal{O}(X)^{op} \to \textbf{Set}$  such that for each open $U$ we define $F(U) = \{f: U \to \sR | f \mbox{ is continuous} \}$ and for each $U \subseteq V$ -- which corresponds to the arrow $V \to U$ in $\mathcal{O}(X)^{op}$ -- we have that $f\in F(V)$ is sent to its restriction $f_{|_U} \in F(U)$. This functor is the presheaf of continuous functions and, actually, it is a classic example of a sheaf. 
\end{example}

\begin{definition}
    The equalizer of a pair of morphism 
\begin{tikzcd}[ampersand replacement=\&]
	A \& B
	\arrow["f", shift left=2, from=1-1, to=1-2]
	\arrow["g"', shift right=2, from=1-1, to=1-2]
\end{tikzcd}  of \textbf{C} consists of an object $Eq(f,g)$ and a morphism $e: Eq(f,g) \to A$ in \textbf{C} satisfying $f \circ e = g \circ e$ and that given an object $Z$ and any morphism $z: Z \to A$ of \textbf{C} satisfying $f \circ z = g \circ z$ then we have a unique morphism $u: Z \to E$ such that $e \circ u = z$. This definition is represented by the following diagram:
\[\begin{tikzcd}[ampersand replacement=\&]
	Z \\
	E \& A \& B
	\arrow["u"', dashed, from=1-1, to=2-1]
	\arrow["z", from=1-1, to=2-2]
	\arrow["e"', from=2-1, to=2-2]
	\arrow["f", shift left=2, from=2-2, to=2-3]
	\arrow["g"', shift right=2, from=2-2, to=2-3]
\end{tikzcd}\]

\end{definition}

\begin{example}
    In \textbf{Vect}, the kernel of a morphism $f: A \to B$ is the equalizer of $f$ and the trivial linear map $0: A \to B$.
\end{example}

\begin{proposition}
    The category \textbf{SDef} has equalizers for all pair of morphisms.
\end{proposition}
\begin{proof}
   Given a pair of morphisms 
\begin{tikzcd}[ampersand replacement=\&]
	C \& {C'}
	\arrow["{\varphi_A}", shift left, from=1-1, to=1-2]
	\arrow["{\varphi_B}"', shift right, from=1-1, to=1-2]
\end{tikzcd}, if $\operatorname{Eq}(\varphi_A,\varphi_B) = \{\Sigma \in C \,|\, A \Sigma A^{\top} =   B \Sigma B^{\top}\}$ and  $i: \operatorname{Eq}(\varphi_A,\varphi_B) \to C$ is the inclusion, it is clear that $\varphi_A \circ  i = \varphi_B \circ  i$. 

Given a cone in \textbf{SDef} and a morphism $\varphi_X : Z \to C$ such that $\varphi_A \circ  \varphi_X = \varphi_B \circ  \varphi_X$, we can define a morphism $u : Z \to \operatorname{Eq}(\varphi_A,\varphi_B)$ by $u(\Sigma) = X \Sigma X^{\top}$, for all $\Sigma \in Z$. In this way,  $u(\Sigma) \in  \operatorname{Eq}(\varphi_A,\varphi_B)$ because  $\varphi_A \circ  \varphi_X = \varphi_B \circ  \varphi_X$ and $i \circ u = \varphi_X$. Moreover, $u$ is the unique morphism such that $i \circ u = \varphi_X$. Indeed, suppose $v : Z \to \operatorname{Eq}(\varphi_A,\varphi_B)$ is such that $i \circ v = \varphi_X$, then $v = i \circ v = \varphi_X = i \circ u = u.$
\end{proof}

\begin{definition}
    The product of two objects $A$ and $B$ of \textbf{C} is an object of \textbf{C} denoted by $A \times B$ equipped with projection morphisms $\pi_1 : A \times B \to A $ and $\pi_2 : A \times B \to B $ satisfying that for any object $Z$ equipped with morphisms $z_1: Z \to A$ and $z_2: Z \to B$, there exists a unique morphism $u: Z \to A \times B$ such that $\pi_2 \circ u = z_2$ and $\pi_1 \circ u = z_1$. Diagrammatically:
\[\begin{tikzcd}[ampersand replacement=\&]
	\& Z \\
	A \& {A\times B} \& B
	\arrow["{z_1}"', from=1-2, to=2-1]
	\arrow["u", dotted, from=1-2, to=2-2]
	\arrow["{z_2}", from=1-2, to=2-3]
	\arrow["{\pi_1}", from=2-2, to=2-1]
	\arrow["{\pi_2}"', from=2-2, to=2-3]
\end{tikzcd}\]   
\end{definition}
It is also possible to define product of an arbitrary (not necessarily finite) family. In this case, the product in \textbf{Vect} is the cartesian product with addition defined component-wise, and the scalar multiplication distributing over all the components; thus, is the direct product. Since we are interested only in vector spaces of finite dimension, we can use that the finite direct product  is  isomorphic to the finite direct sum. Then the direct sum is the product for the category of finite vector spaces. In the infinite case, the direct sum is a dual notion called coprodut.

\begin{proposition}
    The category \textbf{SDef} has all finite products.
\end{proposition}
\begin{proof}
    We will show that if $C$ and $C'$ are cones in \textbf{SDef}, then the direct sum $C \oplus C'$ is the product in \textbf{SDef}. The proof is the same if we have a finite direct sum of cones of PSD matrices. 

    Suppose $C$ is formed by PSD matrices in $M_n(\sR)$ and $C'$ is formed by PSD matrices in $M_m(\sR)$. Then
    $$C \oplus C' = \left\{\Sigma_1 \oplus \Sigma_2 = \begin{bmatrix} \Sigma_1 & 0_{n \times m} \\ 0_{m \times n} & \Sigma_2 \end{bmatrix} \:\Big|\: \Sigma_1 \in C, \Sigma_2 \in C'\right\}.$$

    The resulting matrices $\Sigma_1 \oplus \Sigma_2$ are still PSD: if $v \in \real^{n+m}$ then $v^{\top}(\Sigma_1 \oplus \Sigma_2)v = v_1\Sigma_1v_1 + v_2\Sigma_2v_2 \geq 0$, with $v_1 \in \real^n, v_2 \in \real^m$ and $v = v_1 \oplus v_2$. 
    
    The projection morphism $ \varphi_{\pi_1} :  C \oplus C' \to C$ and $\varphi_{\pi_2} :  C \oplus C' \to C $ are given by 
\begin{equation*}
    \begin{split}
        \varphi_{\pi_1}(\Sigma_1 \oplus \Sigma_2) = \begin{bmatrix} I_{n \times n} & 0_{n \times m} \end{bmatrix} (\Sigma_1 &\oplus \Sigma_2) \begin{bmatrix} I_{n \times n} & 0_{n \times m} \end{bmatrix}^{\top} \\
        \varphi_{\pi_2}(\Sigma_1 \oplus \Sigma_2) = \begin{bmatrix} 0_{m \times n} & I_{m \times m} \end{bmatrix} (\Sigma_1 &\oplus \Sigma_2) \begin{bmatrix} 0_{m \times n} & I_{m \times m} \end{bmatrix}^{\top},
    \end{split},
\end{equation*}

If $Z$ is any other cone in \textbf{SDef} equipped with morphisms $\varphi_{X_1} : Z \to C$ and $\varphi_{X_2} : Z \to C'$, define $u: Z \to C \oplus C'$ by $u(\Sigma) = \varphi_{X_1}(\Sigma) \oplus \varphi_{X_2}(\Sigma) $, then $\varphi_{\pi_1} \circ u = \varphi_{X_1}$ and $\varphi_{\pi_2} \circ u = \varphi_{X_2}$. Note that $u$ is a morphism of \textbf{SDef} since 
$$u(\Sigma) = \begin{bmatrix} \varphi_{X_1}(\Sigma) & 0 \\ 0 & \varphi_{X_2}(\Sigma) \end{bmatrix}  = \begin{bmatrix} I & 0 \\ 0 & 0 \end{bmatrix}\begin{bmatrix} \varphi_{X_1}(\Sigma) & 0 \\ 0 & \varphi_{X_2}(\Sigma) \end{bmatrix}\begin{bmatrix} I & 0 \\ 0 & 0 \end{bmatrix}^{\top} + \begin{bmatrix} 0 & 0 \\ 0 & I \end{bmatrix}\begin{bmatrix} \varphi_{X_1}(\Sigma) & 0 \\ 0 & \varphi_{X_2}(\Sigma) \end{bmatrix}\begin{bmatrix} 0 & 0 \\ 0 & I \end{bmatrix}^{\top} $$

Note that such $u$ is unique: if $v: Z \to C \oplus C'$ satisfies $\varphi_{\pi_1} \circ v = \varphi_{X_1}$ and $\varphi_{\pi_2} \circ v = \varphi_{X_2}$. Since $v(\Sigma) \in C \oplus C'$, we can write $v(\Sigma) = v_1(\Sigma) \oplus v_2(\Sigma)$ such that $v_1(\Sigma) \in C$ and $v_2(\Sigma) \in C'$. Thus,
$u(\Sigma) = \varphi_{\pi_1}(v(\Sigma)) \oplus \varphi_{\pi_2}(v(\Sigma)) = v_1(\Sigma) \oplus v_2(\Sigma) = v(\Sigma).$
\end{proof}

\begin{definition}
    An object $T$ is terminal if for every object $A$ in \textbf{C}, there exist precisely one morphism $A \to T$. Dually, an object $I$ is initial if for every object $A$ in \textbf{C}, there exist precisely one morphism $I \to A$. 
\end{definition}

\begin{example}
    In \textbf{Vect}, the trivial vector space is both initial and terminal. In \textbf{Set}, the empty set is the initial object and every singleton is a terminal object. 
\end{example}

\begin{proposition}
The category \textbf{SDef} has terminal and initial object.    
\end{proposition}
\begin{proof}
    The trivial vector space is the terminal and the initial object.  
    Indeed, there is only one map that takes $\psd^n$ to $\{0_{n \times n}\}$, and only one map from the latter space to the former, taking it to itself (notice that the zero matrix is PSD). But $\{0_{n \times n}\}$ is isomorphic to $\{0_{1 \times 1}\}$, which is isomorphic to the trivial vector space $\{0\}$.
\end{proof}
Again, we denote by $\mathcal{O}(X)$ the category associated with the poset  of all open sets of a topological space $X$.  A \textit{presheaf of sets} on $X$ is a functor $F: \mathcal{O}(X)^{op} \rightarrow \textbf{Set}$. Given inclusions $U \subseteq V$, we use $s_{|^V_U}$ (or just $s_{|_U}$) to denote the ``restriction map'' from $F(V)$ to $F(U)$.  

\begin{definition}
A presheaf $F:\mathcal{O}(X)^{op} \rightarrow \textbf{Set}$ is a \textit{sheaf} (of sets) if for any open $U$ and any open cover $U = \bigcup\limits_{i\in I} U_i$ the following diagram is an equalizer in the category \textbf{Set}
\begin{center}
     \begin{tikzcd}
F (U) \arrow[r, "e"] & \prod\limits_{i\in I}F (U_i) \arrow[r, "p", shift left=1 ex] 
\arrow[r, "q"', shift right=0.5 ex]  & {\prod\limits_{(i,j) \in I \times I}F (U_i \cap U_j)}
\end{tikzcd}
 \end{center}
 
 where:
 \begin{enumerate}
     \item $e(t) = \{t_{|_{U_i}} \enspace | \enspace i \in I\}, \enspace t \in F (U)$ 
     \item     $p((t_k)_{ k \in I}) = (t_{i_{|_{U_i \cap U_j}}})_{(i,j)\in I\times I}$ \\ $q((t_k)_{k \in I}) = (t_{j_{|_{U_i \cap U_j}}})_{(i,j)\in I\times I}, \enspace (t_k)_{k \in I} \in \prod\limits_{k\in I}F (U_k)$
 \end{enumerate}\label{sheaf}
\end{definition}
The above equalizer reads as follows: if $t_i \in F(U_i)$  is a \textit{compatible family}, i.e., $t_{i_{|_{U_i \cap U_j}}} = t_{j_{|_{U_i \cap U_j}}}$ for all $i,j \in I$, there is a unique  $t \in F(U)$ such that $t_{|_{U_i}} = t_i, i \in I$.

This definition makes sense for sets because \textbf{Set} has equalizers for all pairs of morphisms and products for a collection of sets, provided the collection is itself a set\footnote{Saying that the collection is a set is implicitly given us a notion of size. In general, we need \textit{small} products but we prefer to not address details about it here, since the results we prove are for finite vector spaces and then we only consider finite products.}. In other words, \textbf{Set} is a complete category. Then this definition makes sense for any complete category, such as \textbf{Vect}. 

The terminal object in \textbf{Set} plays an implicit role in the definition of a sheaf: take $U$ as the empty subspace of $X$ and $I$ the empty set of indices. Then a necessary condition for a presheaf $F$ be a sheaf is that $F(\emptyset)$ is a singleton.

We already argued that the trivial vector space is the terminal and initial object in \textbf{SDef}. Consequently, we can define sheaves of positive semidefinite matrices. 

Now, we briefly relate sheaves on a topological space with cellular sheaves. 

Using $\Bar{X}$ to denote the closure of $X$:
\begin{definition} \cite{curry2014sheaves}
    A cell complex is a space $X$ with a partition into pieces $\{X_\sigma\}_{\sigma \in P_X}$ satisfying:
    \begin{enumerate}
        \item  Each point $x \in X $ has an open neighborhood $ U$ intersecting only finitely many $X_\sigma$.
        \item $X_\sigma$ is homeomorphic to $\sR^k$, for some $k$  (where $\sR^0$ is one point).
        \item If $\Bar{X_\tau} \cap X_\sigma$ is non-empty, then $X_\sigma \subseteq \Bar{X_\tau}$.  When this occurs we say the pair are incident or that $X_\sigma$ is a face of $ X_\tau$. The face relation makes the
indexing set $P_X$ into a poset by declaring $\sigma \leq \tau$.
    \end{enumerate}
\end{definition}

\begin{definition}  \cite{curry2014sheaves}
    A cellular sheaf $F$ valued in $\textbf{D}$ on a cell complex $X$ is a functor $F: P_X \to \textbf{D}$. 
\end{definition}

Note that a graph $G = (V,E)$ is a cell complex of dimension $1$, where the set of indices is $P_G = V \cup E$. Explicitly, 
\begin{equation}
    x \trianglelefteq y \iff
        \begin{cases}
            x \in y, & \text{if } x \in V, y \in E \\
            x = y, & \text{otherwise}
        \end{cases}
\end{equation}
is the preorder over $P_G$, providing that $(P_G, \unlhd)$ is a poset.  We can now define an Alexandrov topology on this poset:
\begin{equation}
    \tau=\{U \subseteq P_G: \forall x, y \in P_G, \: (x \in U \land x \trianglelefteq y) \implies y \in U\},
\end{equation}

for which the following sets (called open stars at $x$) form a basis:
\begin{equation}
    U_x = \{ y \in P_G \:|\: x \trianglelefteq y\},
\end{equation}

It is easily seen that for two open stars $U_x$ and $U_y$, we have
$U_x \subseteq U_y \iff y \trianglelefteq x$.

Finally, 
\begin{proposition}
    Let \textbf{D} be a complete category and $(X, P_X)$ be a cell complex. A cellular sheaf on $X $ is a sheaf on
    $P_X$ equipped with the Alexandrov topology. Such a sheaf is uniquely determined by a functor $F : P_X \to \textbf{D}$.
\end{proposition}
\begin{proof}
    This is the Corollary 4.2.13 in \citet{curry2014sheaves}.
\end{proof}

Therefore, a sheaf of \textbf{SDef} on $P_G$  is uniquely determined by a cellular sheaf of \textbf{SDef} on a graph $G$, as desired.

\section{Cohomology and the sheaf Laplacian} \label{sec:app:cohomology}
Rather than using a coboundary map $\delta$ we had to consider a pair $\{\delta^+,\delta^-\}$ to reason about cohomology for a cellular sheaf valued in \textbf{SDef}, because $\ker\delta$ in this context would be always trivial and, therefore, uninteresting. In this section, we clarify in which sense $\{\delta^+,\delta^-\}$ is a coboundary map.

We recall that a cochain complex in \textbf{Vect} is a sequence of vector spaces $\{C^i\}$ equipped with linear maps $\delta^i : C^i \to C^{i+1}$ such that $\delta^{i+1}\circ \delta^i = 0$. These $\delta^i$ are called coboundary maps. 
\[\begin{tikzcd}[ampersand replacement=\&]
	{C^0} \& {C^1} \& {C^2} \& {C^3} \& \dots
	\arrow["{\delta^0}", from=1-1, to=1-2]
	\arrow["0"', curve={height=12pt}, from=1-1, to=1-3]
	\arrow["{\delta^1}", from=1-2, to=1-3]
	\arrow["0", curve={height=-18pt}, from=1-2, to=1-4]
	\arrow["{\delta^2}", from=1-3, to=1-4]
	\arrow["{\delta^3}", from=1-4, to=1-5]
\end{tikzcd}\]

The condition $\delta^{i+1}\circ \delta^i = 0$ guarantees $Im(\delta^{i-1}) \subset \ker(\delta^i)$ and we define the cohomology groups by $H^i(C^{\bullet}) = \dfrac{\ker(\delta^i)}{Im(\delta^{i-1})}$, where $C^{\bullet}$ denotes a cochain complex. Considering that $\delta^{-1} = 0$, we have that  $H^0(C^{\bullet}) \!=\! \ker(\delta^0)$. 
Suppose we can write $\delta^i = (\delta^i)^+ - (\delta^i)^- $, for all $i \in \sN.$ Then
$ 0 = \delta^{i+1}\circ \delta^i = ((\delta^{i+1})^+ - (\delta^{i+1})^-)\circ ((\delta^i)^+ - (\delta^i)^-).$ Each linear map $\delta^i$ have an associated matrix, which we denote by the same symbol $\delta^i$ and then composition is just multiplication of matrices.  So:
$$0 =  (\delta^{i+1})^+ (\delta^{i})^+ + (\delta^{i+1})^-(\delta^{i})^- - (\delta^{i+1})^-(\delta^{i})^+ - (\delta^{i+1})^+(\delta^{i})^-$$
While this equation holds for general matrices, it does not work for PSD matrices. However, there is an equivalent way of stating this equation that is suitable for PSD matrices:
\begin{definition}
    A cochain complex in \textbf{SDef} is a sequence of cones of positive semidefinite matrices $\{C^i\}$ equipped with morphisms $(\delta^i)^+,(\delta^i)^- : C^i \to C^{i+1}$ such that
    $$(\delta^{i+1})^+ (\delta^{i})^+ + (\delta^{i+1})^-(\delta^{i})^- = (\delta^{i+1})^-(\delta^{i})^+ + (\delta^{i+1})^+(\delta^{i})^-$$
    Then we say the pair $\{(\delta^i)^+,(\delta^i)^-\}$ is a coboundary map.
\end{definition}

Note that a cochain complex of semimodules over semirings is defined in the same way in \citet{patchkoria2006exactness}. The definition of the cohomology groups is then given by $H^i(C^\bullet) = \dfrac{Eq((\delta^i)^+,(\delta^i)^-)}{\rho^i}$, where $\rho^i$ is a congruence relation on $Eq((\delta^i)^+,(\delta^i)^-)$. Although positive semidefinite matrices are not semimodules, both structures present the same failure: the absence of an inverse of the addition. The above consideration explains why  we translated constructions in  \citet{patchkoria2006exactness} to explore basic cohomological constructions in \textbf{SDef}.

In the case of cohomology for a cellular sheaf $F$ over a graph $G$, there are only two non-trivial vector spaces in our cochain complex: 
\[\begin{tikzcd}[ampersand replacement=\&]
	{C^0(G,F)} \& {C^1(G,F)} \& 0 \& 0 \& \dots
	\arrow["{\delta^0 = \delta}", from=1-1, to=1-2]
	\arrow["0"', curve={height=12pt}, from=1-1, to=1-3]
	\arrow["0", from=1-2, to=1-3]
	\arrow["0", curve={height=-18pt}, from=1-2, to=1-4]
	\arrow["0", from=1-3, to=1-4]
	\arrow["0", from=1-4, to=1-5]
\end{tikzcd}\]

In other words, $C^i = \bigoplus_{v \in V} F(i-cells)$, for all $i \geq 0$, $\delta^i = 0$ for all $i \neq 0$, and $\delta^0 = \delta$ is the coboundary map defined in Section \ref{sec:prelim}. 

 For clarity, observe that if we had a cell complex $X$ with faces connecting edges, then the sequence would be of the form 

\[\begin{tikzcd}[ampersand replacement=\&]
	{C^0(X,F)} \& {C^1(X,F)} \& {C^2(X,F)} \& 0 \& \dots
	\arrow["\delta", from=1-1, to=1-2]
	\arrow["0"', curve={height=12pt}, from=1-1, to=1-3]
	\arrow["{\delta^1}", from=1-2, to=1-3]
	\arrow["0", curve={height=-18pt}, from=1-2, to=1-4]
	\arrow["0", from=1-3, to=1-4]
	\arrow["0", from=1-4, to=1-5]
\end{tikzcd}\]

Since we work only with sheaves over a graph, we only have to worry with $\delta^0$, thus, we can omit the index. Therefore, for cellular  sheaves valued in \textbf{SDef}, it is enough to study $H^0(C^\bullet) = \dfrac{Eq(\delta^+,\delta^-)}{\rho^0}$, and $x \rho ^0 y$ iff $x = y$. A calculation shows that  $H^0(C^\bullet) = Eq((\delta^0)^+,(\delta^0)^-)$ \citep{JUN2017306}. 

Since for graphs $(\delta^i)^+ = 0 = (\delta^i)^-$, for all $i \neq 0$, the maps for the sheaf of covariances defined as $\delta_{\gC}^+(\Sigma) = \sum_i \varphi_{\mathbf{A}_i}(\varphi_{\delta_{\gM}^+}(\Sigma))$ and $\delta_{\gC}^-(\Sigma) = \sum_i \varphi_{\mathbf{A}_i}(\varphi_{\delta_{\gM}^-}(\Sigma))$ form a coboundary map in \textbf{SDef}, where $\delta = \delta^0$. 

Now, recall that if $\gF$ is \textbf{Vect}-valued cellular sheaf on a graph $G$, the sheaf Laplacian of $\gF$ is defined by $L_{\gF} = \delta_{\gF}^{\top}\delta_{\gF}$, with $\delta_{\gF} : C^0(G,\gF) \to C^1(G,\gF)$ the correspondent coboundary map.  However, in \textbf{SDef}, we have a pair $\{\delta^+,\delta^-\}$ playing the role of $\delta$. We could repeat the same idea and break the Laplacian, in \textbf{Vect}, $\delta_{\gF} = \delta^+_{\gF} - \delta^-_{\gF}$. Then $$L_{\gF} = \delta^{+^{\top}}_{\gF}\delta^+_{\gF} + \delta^{-^{\top}}_{\gF}\delta^-_{\gF} - (\delta^{+^{\top}}_{\gF}\delta^-_{\gF} + \delta^{-^{\top}}_{\gF}\delta^+_{\gF}) = L_{\gF}^+ - L_{\gF}^-.$$ Again, this works in \textbf{Vect} but $L$ defined in this way is not enough to build a good operator in \textbf{SDef}, we need a nice behavior locally and the preservation of the block structure. Thus we defined the Laplacian in an alternative way aiming to keep the use of the coboundary maps: define a map $\delta_{\gC}(\Sigma) = \delta_{\gC}^+(\Sigma) + \delta_{\gC}^-(\Sigma)$  and then the Laplacian by $L_{\gC}(\Sigma) = \sum_{i=1}^n \varphi_{\mathbf{B}_i}(\varphi_{(\delta_{\gM}^+)^{\top}}(\delta_{\gC}(\Sigma)) + \varphi_{(\delta_{\gM}^-)^{\top}}(\delta_{\gC}(\Sigma)))$, with $\mathbf{B}_i \in \sR^{nd \times nd}$ block-diagonal matrix whose $i$-th block equals $I_{d \times d}$ and all others are $0_{d \times d}$. Notice that
\begin{align*}
    L_{\gC}(\Sigma) 
    &= \sum_{i=1}^n \varphi_{\mathbf{B}_i}(\varphi_{(\delta_{\gM}^+)^{\top}}(\delta_{\gC}^+(\Sigma) + \delta_{\gC}^-(\Sigma)) + \varphi_{(\delta_{\gM}^-)^{\top}}(\delta_{\gC}^+(\Sigma) + \delta_{\gC}^-(\Sigma))) \\ 
    &= \sum_{i=1}^n \varphi_{\mathbf{B}_i}\left(\sum_j \varphi_{(\delta_{\gM}^+)^{\top}\mathbf{A}_j\delta_{\gM}^+}(\Sigma) + \varphi_{(\delta_{\gM}^+)^{\top}\mathbf{A}_j\delta_{\gM}^-}(\Sigma) + 
    \varphi_{(\delta_{\gM}^-)^{\top}\mathbf{A}_j\delta_{\gM}^+}(\Sigma) + \varphi_{(\delta_{\gM}^-)^{\top}\mathbf{A}_j\delta_{\gM}^-}(\Sigma)\right),
\end{align*}

i.e. we are using the morphisms $\varphi$ induced by $\{\delta^{+^{\top}}_{\gF}\delta^+_{\gF}, \delta^{-^{\top}}_{\gF}\delta^-_{\gF}, \delta^{+^{\top}}_{\gF}\delta^-_{\gF}, \delta^{-^{\top}}_{\gF}\delta^+_{\gF}\}$, with the addition of correcting factors $\mathbf{A}_j$ and $\mathbf{B}_i$. This operator together with the sheaf Laplacian of sheaf of means in \textbf{Vect} provided the Laplacian for the Gaussian sheaf in $\textbf{Vect}\times \textbf{SDef}$, which recovered results analogous to those obtained in \citet{bodnar2023topological}. For instance: it generalizes the graph Laplacian; the equalizer of the positive and negative parts of the Laplacian coincides with $H^0(G,\gF)$; under suitable conditions, any Gaussian sheaf can be learned; and the adequate notion of energy is zero only when the distribution is an element in the global section of the Gaussian sheaf.

\section{An alternative description of the Gaussian sheaf}\label{sec:app:gauss_sheaf}
Now, we explicitly describe the Gaussian sheaf and the GSNN model in terms of the distributions intead of their parameters (mean and covariance). First, consider the product category $\textbf{Vect} \times \textbf{SDef}$, whose objects are pairs of vector spaces and PSD cones $(V, \psd^n)$, and morphisms $(A, \varphi_B)$. Since both categories are complete, their product is again complete. We are particularly interested in a subcategory $\gD$ of this product, whose objects takes the form $(\sR^n, \psd^n)$ and with morphisms restricted to $\Phi_A = (A, \varphi_A)$. Notice that $\Phi_B \circ \Phi_A = (BA, \varphi_{BA}) = \Phi_{BA}$. 

Let $\gauss{d}$ be the space of Gaussian distributions on $\sR^d$. The structure preserving operation in this space is the pushforward by linear maps (whose representation are square matrices), i.e. for a Gaussian distribution $\nu$ with parameters $(\mu, \Sigma)$, $A\#\nu$ is again Gaussian with parameters $(A\mu, A\Sigma A^{\top})$, for $A \in M_n(\sR)$ (notice that the parameters of the Gaussian $A\#\nu$ are $\Phi_A(\mu,\Sigma)$). We can also go to any other space $\gauss{m}$ pushing forward the distribution by any $B \in \operatorname{Hom}(\sR^n, \sR^m)$. The pushforward by the identity map is the identity operation, and $B\#(A\#\nu) = (BA)\#\nu$ for suitable maps $A$ and $B$. Therefore, we have a category \textbf{Gauss} whose objects are spaces of Gaussian distributions and morphisms are pushforward by linear maps (the remaining desired properties follows directly from properties of linear maps). Notice that when we described this category we only cared about how the parameters of a given Gaussian distribution changed under the action of the pushforward. This is due to the bijection we mentioned before, i.e. $\gauss{d} \simeq \sR^n \times \psd^n$. This gives an isomorphism of categories $\textbf{Gauss} \simeq \gD$. 

The Gaussian sheaf we defined before was formed by the sheaf of means, which was valued on $\textbf{Vect}$, and by the sheaf of covariances, which was valued on $\textbf{SDef}$. In other words, the Gaussian sheaf was valued in $\textbf{Vect} \times \textbf{SDef}$. Equivalently, we can define the following Gaussian sheaf with values on \textbf{Gauss}.
\begin{definition}
    Given an undirected graph $G=(V,E)$, a Gaussian (cellular) sheaf $(G, \gF)$ associates:
    \begin{itemize}
        \item A space of Gaussian distributions $\gF(v) \coloneqq \gauss{d_v}$ for every vertex $v \in V$.
        \item A space of Gaussian distributions $\gF(e) \coloneqq \gauss{d_e}$ for every edge $e \in E$.
        \item A morphism $\Phi_{\rmapf{v}{e}}: \gF(v) \to \gF(e)$ to every incident vertex-edge pair $v \unlhd e$, the pushforward by the linear map $\rmapf{v}{e}$.
    \end{itemize}
\end{definition}

The isomorphism between the categories provides this sheaf is categorically equivalent to the sheaf we defined in terms of means and covariances separately.  

In terms of distributions, we point out that a 0-cochain $\nu \in C^0(G,\gF)$ is $\nu = (\nu_{v_1}, \dots, \nu_{v_n})$ with $\nu_{v_i} = \mathcal{N}(\mu_{v_i}, \Sigma_{v_i})$. We have that $\nu = \bigotimes\limits_{i \in [n]} \mathcal{N}(\mu_{v_i}, \Sigma_{v_i})$, therefore $\nu$ is also Gaussian, with first moment $\mu = \bigoplus\limits_{i \in [n]} \mu_{v_i} \in \mathbb{R}^{nd}$ and covariance operator $\Sigma = \bigoplus\limits_{i \in [n]} \Sigma_{v_i}$, i.e. $\nu$ is induced by the product measure of the measures associated to each of its components.

We define the convolution of the coboundaries and Laplacian on this new sheaf.
\begin{definition}\label{def:coboundary_laplacian_distributions}
    Let $(G,\gF)$ be a Gaussian sheaf defined in terms of distributions. Given a 0-cochain $\nu \in C^0(G,\gF)$, we define
    \begin{footnotesize}
        \begin{align}
        \delta_{\gF}(\nu) = \Conv_{i=1}^{|E|} \mathbf{A}_i\#[(\delta_{\gM}^+\#\nu) &\ast (-\delta_{\gM}^-)\#\nu] \\
        L_{\gF}(\nu) = \Conv_{j=1}^n \mathbf{B}_j\#[(\delta_{\gM}^+)^{\top}\#\delta_{\gF}(\nu) &\ast (-\delta_{\gM}^-)^{\top}\#\delta_{\gF}(\nu)],
    \end{align}
    \end{footnotesize}
    where $\Conv$ denotes the convolution operator between the obtained distributions for each index.
\end{definition}
The next result illustrates the connection of the two definitions of the Gaussian sheaf through the application of the $\delta_{\gF}$ and $L_{\gF}$ defined above.
\begin{corollary}
\label{cor:gausslap}
    For a 0-cochain $\nu \in C^0(G,\gF)$ with parameters $\mu$ and $\Sigma$, we have $\delta_{\gF}(\nu) = \gN(\delta_{\gM}\mu, \delta_{\gC}(\Sigma))$ and $L_{\gF}(\nu) = \gN(L_{\gM}\mu, L_{\gC}(\Sigma))$.
\end{corollary}
\begin{proof}
    Using definition \ref{def:coboundary_laplacian_distributions}, we have
\begin{equation*}
    \Conv_i \mathbf{A}_i\#(\delta_{\gM}^+\#\nu) = \mathcal{N}\left(\sum_i \mathbf{A}_i\delta_{\gM}^+\mu, \sum_i \varphi_{\mathbf{A}_i}(\varphi_{\delta_{\gM}^+}\Sigma)\right) = \gN\left(\left(\sum_i \mathbf{A}_i\right)\delta_{\gM}^+\mu, \delta_{\gC}^+(\Sigma)\right) = \gN(\delta_{\gM}^+\mu, \delta_{\gC}^+(\Sigma)),
\end{equation*}

since $\sum_i \mathbf{A}_i = I_{|E|d \times |E|d}$. Also $\Conv_i \mathbf{A}_i\#((-\delta_{\gM}^-)\#\nu) = \gN(-\delta_{\gM}^-\mu, \delta_{\gC}^-(\Sigma))$. Therefore
\begin{equation*}
    \delta_{\gF}(\nu) = \Conv_i \mathbf{A}_i\#(\delta_{\gM}^+\#\nu) \ast \Conv_i \mathbf{A}_i\#(\delta_{\gM}^-\#\nu) = \gN(\delta_{\gM}^+\mu - \delta_{\gM}^-\mu, \delta_{\gC}^+(\Sigma)+\delta_{\gC}^-(\Sigma)) = \gN(\delta_{\gM}\mu, \delta_{\gC}\Sigma).
\end{equation*}
Following a similar computation we can conclude that $L_{\gF}(\nu) = \gN(L_{\gM}\nu, L_{\gC}(\Sigma))$.
\end{proof}

The normalized Laplacian $\Delta_{\gF}$ can be achieved through $\Delta_{\gF}(\nu) = \Conv_i^k \Delta_i\#\nu \ast \Conv_j \mathbf{B}_j\#\Delta'\#\nu$, where $\Delta_i = D^{-\frac{1}{2}}L_iD^{-\frac{1}{2}}$, $\Delta' = D^{-\frac{1}{2}}L'D^{-\frac{1}{2}}$, and $L_{\gM} = \sum_i^k L_i + L'$ is the decomposition mentioned before. It is also useful to define
\begin{equation}
    (I - \Delta_{\gF})(\nu) = \Conv_{i=1}^k ((k+1)^{-1}I - \Delta_i)\#\nu \ast \Conv_{j=1}^n \mathbf{B}_j\#((k+1)^{-1}I - \Delta')\#\nu,
\end{equation}
whose resulting distribution is $(I - \Delta_{\gF})(\nu) = \gN((I - \Delta_{\gM})\mu, (I - \Delta_{\gC})(\Sigma))$, with $(I - \Delta_{\gC})(\Sigma)$ being defined accordingly to the operations described above.

Finally, recall that \cref{sec:GSNN} introduced our GSNN in terms of means and covariances. The alternative way of presenting the GSNN model is the following:
\begin{equation}\label{eq:model}
    \nu' = (I - \Delta_{\gF})^{\ell}((I \otimes W_1)\#f_{W_2}(x)\#\nu),
\end{equation}
with $W_1 \in \mathbb{R}^{d \times d}$, $W_2 \in \mathbb{R}^{h \times h}$ weight matrices, $\ell$ the number of layers, and $f_{W_2}(x) = (\mu W_2 + (x - \mu))$ an affine map that mixes the mean vectors in the channels.

\begin{table*}[h]
\centering
\caption{p-values from one-sided Wilcoxon Signed-Rank tests computed over 10 independent runs for each method and dataset. 
GSNN models consistently outperform the baselines at a 5\% significance level across real and synthetic datasets. The only exception is the Watts--Strogatz ($k=45$) dataset, where the best-performing GSNN is statistically indistinguishable from the best NSD method.}
\label{tab:p_values}
\footnotesize
\setlength{\tabcolsep}{4pt}

\begin{minipage}[t]{0.48\linewidth}
    \centering
    \begin{tabular}{llcc}
    \toprule
    Dataset & Method 1 & Method 2 & p-val \\
    \midrule
    \multirow{6}{*}{BA ($m=25$)} 
     & GCN       & MLP      & 0.712 \\
     & Gen-NSD   & MLP      & 0.070 \\
     & Gen-NSD   & GCN      & 0.166 \\
     & O(d)-GSNN & Gen-NSD  & 0.055 \\
     & O(d)-GSNN & MLP      & 0.003 \\
     & O(d)-GSNN & GCN      & 0.005 \\
    \midrule
    \multirow{6}{*}{BA ($m=50$)} 
     & GCN       & MLP      & 0.997 \\
     & Gen-NSD   & MLP      & 0.005 \\
     & Gen-NSD   & GCN      & 0.003 \\
     & Gen-GSNN  & Gen-NSD  & 0.014 \\
     & Gen-GSNN  & GCN      & 0.003 \\
     & Gen-GSNN  & MLP      & 0.003 \\
    \bottomrule
    \end{tabular}
\end{minipage}
\hfill
\begin{minipage}[t]{0.48\linewidth}
    \centering
    \begin{tabular}{llcc}
    \toprule
    Dataset & Method 1 & Method 2 & p-val \\
    \midrule
    \multirow{6}{*}{WS ($k=25$)} 
     & GCN       & MLP      & 0.254 \\
     & O(d)-NSD  & GCN      & 0.003 \\
     & O(d)-NSD  & MLP      & 0.003 \\
     & Gen-GSNN  & O(d)-NSD & 0.046 \\
     & Gen-GSNN  & GCN      & 0.003 \\
     & Gen-GSNN  & MLP      & 0.003 \\
    \midrule
    \multirow{6}{*}{WS ($k=45$)} 
     & GCN       & MLP      & 0.254 \\
     & Diag-NSD  & GCN      & 0.003 \\
     & Diag-NSD  & MLP      & 0.003 \\
     & Gen-GSNN  & Diag-NSD & 0.600 \\
     & Gen-GSNN  & GCN      & 0.003 \\
     & Gen-GSNN  & MLP      & 0.003 \\
    \bottomrule
    \end{tabular}
\end{minipage}

\vspace{15pt} 

\begin{minipage}[t]{0.48\linewidth}
    \centering
    \begin{tabular}{llcc}
    \toprule
    Dataset & Method 1 & Method 2 & p-val \\
    \midrule
    \multirow{6}{*}{Weather1} 
     & GCN       & MLP      & 0.101 \\
     & Diag-NSD  & GCN      & 0.003 \\
     & Diag-NSD  & MLP      & 0.003 \\
     & O(d)-GSNN & Diag-NSD & 0.006 \\
     & O(d)-GSNN & GCN      & 0.003 \\
     & O(d)-GSNN & MLP      & 0.003 \\
    \midrule
    \multirow{6}{*}{Weather2} 
     & GCN       & MLP      & 0.998 \\
     & Gen-NSD   & GCN      & 0.893 \\
     & Gen-NSD   & MLP      & 0.998 \\
     & Gen-GSNN  & Gen-NSD  & 0.002 \\
     & Gen-GSNN  & GCN      & 0.002 \\
     & Gen-GSNN  & MLP      & 0.002 \\
    \bottomrule
    \end{tabular}
\end{minipage}
\end{table*}

\section{Statistical Relevance}\label{app:stat_relevance}

While some datasets show no significant difference among GSNN models, our results are statistically better than other models. To confirm this, we ran Wilcoxon Signed-Rank tests. This test verifies whether the distribution of one sample is stochastically smaller than the distribution of another sample. 
Importantly, this is a standard technique for assessing whether the performances of independent runs of competing algorithms on a given task are statistically distinguishable. We present the p-values for each pair of best-performing methods in \cref{tab:p_values}. GSNN performs better than baselines at a 5\% level of statistical significance. The only exception is the W-S (k=45), in which GSNN's performance is statistically indistinguishable from Diag-NSD.

\section{Runtime}\label{app:runtime}
To demonstrate that GSNN does not drastically increases the runtime of a standard GNN such as GCN, we provide \cref{tab:runtime}. We acknowledge the scalability challenges of sheaf NNs; however, we also understand that the results are of the same magnitude order. Thus, any dataset that operates within a reasonable time-frame in GCN can be tested with a GSNN. Furthermore, the scalability issue is an intrinsic aspect of the current stage of sheaf models, as a nascent paradigm, representing a key area of ongoing research. 

\begin{table}[htbp]
\centering
\caption{Average training time per epoch (seconds) across datasets. All tests were performed on a NVIDIA GeForce RTX 3090. The results are of the same magnitude order.}
\vspace{2pt}
\label{tab:runtime}
\begin{tabular}{lcccccc}
\toprule
Model & BA (m=25) & BA (m=50) & WS (k=25) & WS (k=45) & Weather1 & Weather2 \\
\midrule
GCN     & 0.2168 & 0.2227 & 0.2185 & 0.2201 & 0.8829 & 0.8357 \\
MLP     & 0.2195 & 0.2201 & 0.2196 & 0.2215 & 1.0647 & 0.8126 \\
GenGSNN & 0.9579 & 1.1704 & 0.6223 & 0.7370 & 1.5082 & 1.5196 \\
GenNSD  & 0.3418 & 0.3703 & 0.3190 & 0.3421 & 0.9810 & 0.9898 \\
\bottomrule
\end{tabular}
\end{table}

\section{Model details and hyperparameters}\label{sec:app:model}

\textbf{On the choice of the 2-Wasserstein metric.} Beyond its role as a generalization of the MSE to probability distributions, the 2-Wasserstein metric offers two further advantages exploited by GSNN. First, unlike common divergence measures such as Kullback–Leibler or R\'enyi divergences, it is less sensitive to tail discrepancies, which yields stabler training and avoids the need for an explicit parameterization of the underlying distributions. Second, both the metric and its gradients can be efficiently computed through the Sinkhorn algorithm, enabling scalable stochastic gradient-based optimization.

\textbf{Sheaf-learning details.} The sheaf-learning process is to learn the restriction maps using the features of the graph. Depending on the type of problem, one might be interested in using another version of \cref{eq:gauss_rmap}. In our experiments, we found it useful to feed the determinant of covariances to the MLP $\Psi$ instead of the whole matrix. Another option is to use the vectorized lower triangular matrix $vech(\Sigma_{\nu_v})$ as input.

In \cref{tab:hyp} we show the hyperparameters used for the grid search in the experiments. We train the models for 1500 epochs, with a fixed learning rate, and perform early stopping and learning rate reduction after 100 and 20 epochs of no improvement, respectively. We fix the stalk dimension in 2, since the input distributions are all bi-dimensional.
\begin{table}[htbp]
    \centering
    \begin{tabular}{cc}
        \toprule
        Stalk dimension &  2 \\
        \midrule
        Layers & [1, 2] \\
        \midrule
        Hidden channels & [32, 64] \\
        \midrule
        MLP (restriction maps) & 1 layer of size [32, 64]  \\
        \midrule
        MLP (readout) & 1 layer of size 32 (synth) and 64 (real) \\
        \midrule
        Weight Decay & 5e-3 \\
        \midrule
        Sheaf Decay & 5e-3 \\
        \bottomrule
    \end{tabular}
    \caption{Hyperparameters for the experiments}
    \label{tab:hyp}
\end{table}

\end{document}